\documentclass[letterpaper]{article} 
\usepackage{aaai23}  
\usepackage{times}  
\usepackage{helvet}  
\usepackage{courier}  
\usepackage[hyphens]{url}  
\usepackage{graphicx} 
\usepackage{natbib}  
\usepackage{caption} 
\usepackage{algorithm}

\usepackage[utf8]{inputenc} 

\usepackage{url}            
\usepackage{booktabs}       
\usepackage{amsfonts}       
\usepackage{nicefrac}       
\usepackage{microtype}      

\usepackage{microtype}
\usepackage{graphicx}
\usepackage{algpseudocode}
\usepackage{booktabs} 
\usepackage{amsmath}
\usepackage{amsthm}
\usepackage{amssymb} 

\usepackage{enumerate}
\usepackage{bm}
\usepackage{graphicx}
\usepackage{thmtools}
\usepackage{thm-restate}
\usepackage{mathtools}
\usepackage[skip=0pt]{caption}
\usepackage[skip=0pt]{subcaption}

\usepackage{booktabs}
\usepackage{fontawesome}
\usepackage{enumitem}

\newtheorem{lemma}{\bf Lemma}
\newtheorem{fact}{\bf Fact}

\newtheorem{definition}{\bf Definition}

\usepackage{newfloat}
\usepackage{listings}
\DeclareCaptionStyle{ruled}{labelfont=normalfont,labelsep=colon,strut=off} 
\floatstyle{ruled}
\newfloat{listing}{tb}{lst}{}
\floatname{listing}{Listing}
\title{Understanding Representation Learnability of Nonlinear Self-Supervised Learning}
\author{
    Ruofeng Yang\textsuperscript{\rm 1}, Xiangyuan Li\textsuperscript{\rm 1}, Bo Jiang\textsuperscript{\rm 1}, Shuai Li\textsuperscript{\rm 1}\thanks{Corresponding author}\\
}
\affiliations{
    \textsuperscript{\rm 1}Shanghai Jiao Tong University\\

    wanshuiyin@sjtu.edu.cn, lixiangyuan19@sjtu.edu.cn,  bjiang@sjtu.edu.cn, shuaili8@sjtu.edu.cn
}

\usepackage{bibentry}

\begin{document}

\maketitle

\begin{abstract}
Self-supervised learning (SSL) has empirically shown its data representation learnability in many downstream tasks. There are only a few theoretical works on data representation learnability, and many of those focus on final data representation, treating the nonlinear neural network as a ``black box". However, the accurate learning results of neural networks are crucial for describing the data distribution features learned by SSL models. Our paper is the first to analyze the learning results of the nonlinear SSL model accurately. We consider a toy data distribution that contains two features: the label-related feature and the hidden feature. Unlike previous linear setting work that depends on closed-form solutions, we use the gradient descent algorithm to train a 1-layer nonlinear SSL model with a certain initialization region and prove that the model converges to a local minimum. Furthermore, different from the complex iterative analysis, we propose a new analysis process which uses the \textbf{exact version of  Inverse Function Theorem} to accurately describe the features learned by the local minimum. With this local minimum, we prove that the nonlinear SSL model can capture the label-related feature and hidden feature at the same time. In contrast, the nonlinear supervised learning (SL) model can only learn the label-related feature. We also present the learning processes and results of the nonlinear SSL and SL model via simulation experiments. 
\end{abstract}

\section{Introduction}\label{sec:introduction}

In recent years, self-supervised learning has become an important paradigm in machine learning because it can use datasets without expensive target labels to learn useful data representations for many downstream tasks \cite{devlin2018bert,radford2019language,wu2020self}. 

At present, contrastive learning, a common self-supervised learning method, has shown superior performance in learning data representations and outperformed supervised learning in some downstream tasks \cite{he2020momentum,DBLP:conf/cvpr/ChenH21,grill2020bootstrap,caron2020unsupervised,WangWSFLJWZL22}. Contrastive learning methods usually form a dual pair of siamese
networks \cite{bromley1993signature} and use data augmentations for each datapoint. They treat two augmented datapoints of the same datapoint as positive pairs and maximize the similarity between positive pairs to learn data representations. However, the siamese networks often collapse to a trivial solution during the training process, rendering the learned representation meaningless.

To avoid the above problem, earlier contrastive learning methods such as MoCo \cite{he2020momentum} and SimCLR \cite{chen2020simple} 
treat augmented datapoints from different datapoints as negative pairs and prevent model collapse by the trade-off between positive and negative pairs. However, obtaining high-quality negative pairs is difficult \cite{khosla2020supervised}, which in turn requires additional changes to the model. Recently, other classes of the SSL model, such as BYOL~\cite{grill2020bootstrap} and SimSiam~\cite{DBLP:conf/cvpr/ChenH21}, which do not use negative pairs, have been studied. These models will not collapse to a trivial solution because they construct subtle asymmetry in the structure of the siamese network and create a dynamic buffer area \cite{tian2021understanding}. SimSiam further simplifies the structure of BYOL and only retains the core asymmetry. The simplified model makes training and analysis more convenient while obtaining competitive and meaningful data representations.

Despite the empirical success of SSL \cite{he2020momentum,chen2020simple,DBLP:conf/cvpr/ChenH21,zhong2022self}, there are only a few works that focus on data representation learnability \cite{arora2019theoretical,tosh2021contrastive,lee2021predicting,haochen2021provable,haochen2022beyond,tian2022deep,tian2022understanding,wen2021toward,liu2021self}. However, studying the learnability is helpful in understanding why SSL models can obtain meaningful data representations. 
Many of the above works used final data representation to study the data representation learnability. \citet{arora2019theoretical} obtained the data representation function by minimizing the empirical SSL loss in a special data representation function class. 
\citet{haochen2021provable} and \citet{haochen2022beyond} studied final data representation by closed-form solutions. They viewed the nonlinear neural network as a ``black box" and ignored the learning result of the nonlinear neural network. Thus their results do not describe the features accurately captured by SSL models and explain the encoding process of neural networks. 

\citet{wen2021toward} and \citet{tian2020understanding} tried to understand the learning results of nonlinear SSL models by analyzing a relatively overparameterized neural network. However, their results do not provide an accurate answer to whether SSL models could exactly capture the important features of data distribution or just capture a mixture of features. 

\citet{liu2021self} studied the learning results of SSL models, and it is the most relevant work to us. They proved that SSL models could learn the label-related features and hidden features at the same time. However, their work is a linear framework, and their results depend on the closed-form solutions of the learning results. When considering a nonlinear SSL model, we can not get closed-form solutions due to the nonconvexity. Therefore, which  features can be exactly learned by nonlinear SSL models remains an important open question. We need a new analysis process to analyze the specific learning results of the nonlinear SSL model.

In this work, for the first time, we use gradient descent to train a \textbf{nonlinear SSL model} and analyze the data representation learnability by using the learning results of neural networks. We accurately describe the data distribution features captured by the SSL model. Specifically, we accomplish:
\begin{enumerate}
    \item With a designed data distribution, we use gradient descent (GD) to train a 1-layer nonlinear SSL model and prove that the model can converge to a local minimum under a certain initialization region. Using locally strong convexity, we also obtain the convergence rate of the algorithm. 
    \item We describe the properties of the local minimum using \textbf{the exact version of Inverse Function Theorem}. Using these properties, we prove that the SSL model learns the label-related feature and hidden feature at the same time. 
    \item We prove that the nonlinear SL model can only learn the label-related feature. In other words, SSL is superior to SL in learning data representation. We verify the correctness of the above results through simulation experiments.
\end{enumerate}

\section{Related Work}\label{sec:related_work}

\paragraph{Theoretical analyses for final data representation.} For the analysis of the data representation learnability, many works focus on the final data representation (the optimal solution of the pretext task) and measure the quality of the final data representation in the downstream tasks by using a linear classifier \cite{haochen2021provable,haochen2022beyond,arora2019theoretical,lee2021predicting,tosh2021contrastive}. The main difference in this line of work is how to obtain the final data representation.  \citet{arora2019theoretical} assumed that the data representation function class contains a function with low SSL loss and minimized the empirical SSL loss in this class.
\citet{haochen2021provable} constructed the population positive-pair graph with augmented datapoints as vertices and the correlation of augmented datapoints as edge weights. Then they proved that the closed-form solutions of the data representations are approximately equivalent to the eigenvectors of the adjacency matrix of the above graph. \citet{lee2021predicting} used the nonlinear canonical correlation analysis (CCA) method to obtain the final data representation. The above works viewed the nonlinear neural network as a ``black box" and ignored the learning results of the neural network. However, the learning results are crucial for analyzing which features are exactly captured by SSL methods. Hence we need to propose a new method to analyze the learning results.

\paragraph{Theoretical analyses for learning results of SSL.}
\citet{liu2021self} analyzed the learning results of SSL methods. With a 1-layer linear SSL model, similar to SimSiam, they demonstrated that the SSL models could learn label-related and hidden features simultaneously. Because of the linear structure and the objective function with a designed quartic regularization, they can directly obtain the closed-form solutions of the learning results by using spectral decomposition of the matrix related to the data distribution. \citet{tian2022deep} and \citet{tian2022understanding} 
dealt with the learning results of the nonlinear SSL model by analyzing an objective function similar to traditional Principal Component Analysis (PCA). However, their results were extended by a hidden neuron. Hence their results can not definitively answer which data features are captured by the model and which are ignored. \citet{wen2021toward} and \citet{tian2020understanding} tried to understand the learning results of the nonlinear SSL by using stochastic gradient descent (SGD). However, their results relied heavily on special data augmentation and relatively overparameterized neural networks. Furthermore, their results only showed that with a large number of neurons, the neural networks contain all data features. They did not accurately characterize the learning result of each neuron. In other words, these results did not 
show the features exactly captured by the SSL methods.

\paragraph{Theoretical guarantees for supervised learning.}
For the analysis of the supervised learning, researchers focus on (1) How to characterize the landscape of the objective function; (2) How to converge to the local minima through algorithms (such as GD and SGD); (3) How fast the algorithm converges to the local minimum \cite{allen2019convergence,du2017gradient,li2017convergence,brutzkus2017globally,du2019gradient}. Hence they focus on characterizing the relationship between the objective function and its gradient and less on the specific form or the properties of local minima. However, the specific forms of local minima are helpful to determine whether SSL methods can capture important data distribution features. 

\section{Problem Formulation}
In this section, we introduce the data distribution and the nonlinear SSL and SL model to be studied in this paper.
\subsection{Data Distribution}
The classification problem is a typical downstream task in machine learning, which can be used to measure the quality of data representation. We start with a simple binary classification and want to explore the differences in the data representations learned by SSL and SL models.

To train models, we first build the data distribution. In most cases, the data distribution contains not only label-related features but also some hidden features. These hidden features may not be helpful for the current task but may be useful for other downstream tasks. We want to determine whether the nonlinear SSL models capture hidden features, resulting in a richer data representation. At the same time, we also wonder whether the SL models only learn label-related features.

For the simplicity of analysis, we consider the label-related features as a group, represented by the feature $e_1$. We also use $e_2$ to represent the hidden features. Inspired by previous work \cite{liu2021self}, which solved the above question in the linear setting, we construct a data distribution containing four kinds of datapoints. The number of these four kinds of datapoints are $n_1, n_2, n_3, n_4$ and $n=n_1+n_2+n_3+n_4$. Every time we generate a datapoint, we draw among the four kinds of datapoints with a probability of $1/4$, which means $\mathbb{E}[n_l] = n/4, \forall l\in[4]$. Let $\tau>1, \rho>0$ are two hyperparameters of the data distribution and $\xi_1,..., \xi_n\in \mathbb{R}^d$ are datapoint noise terms sampled from a Gaussian distribution $\mathcal{N}(0,I)$. Define 
\begin{align}
    \mathcal{D}_1 &= \{x_i|x_i=e_1+\rho \xi_i\}_{i=1}^{n_1}\,,\notag\\
    \mathcal{D}_2 &= \{x_i|x_i=e_1+\tau e_2+\rho \xi_i\}_{i=n_1+1}^{n_1+n_2}\,,\notag\\
    \mathcal{D}_3 &= \{x_i|x_i=-e_1+\rho \xi_i\}_{i=n_1+n_2+1}^{n_1+n_2+n_3}\,,\notag\\
    \mathcal{D}_4 &= \{x_i|x_i=-e_1+\tau e_2+\rho \xi_i\}_{i=n-n_4+1}^{n}\,,\label{dis_SSL}
\end{align}
as the datasets of four kinds of datapoints, where $e_1, e_2 \in \mathbb{R}^d$ are two orthogonal unit-norm vectors. Then, the data distribution in this paper is $\mathcal{D}=\mathcal{D}_1\cup\mathcal{D}_2\cup\mathcal{D}_3\cup\mathcal{D}_4$.

Because labels are required during the SL model training process, we modify the data distribution. Specifically, we denote the class label by $y=\{0,1\}$. When $x_i\in \mathcal{D}_1\bigcup \mathcal{D}_2, y=0$, otherwise $y=1.$ After the above steps, we obtain the data distribution $\mathcal{D}^{\text{SL}}$ of the nonlinear SL model.

It is clear that the binary classification task can be completed using only the representative label-related feature $e_1$. However, since $\tau \geq1$, $e_2$ is also an important hidden feature.

Although this data distribution is a toy setting, it is sufficient to distinguish the learnability of the SSL and SL models. This data distribution is also representative. In Sec 4.2, we explain that the proof process can be easily extended to a more general data distribution containing many label-related and hidden features.

\subsection{Model}\label{subsec:model}
In this section, we introduce the activation function and then the nonlinear SSL and SL model.

In this paper, we analyze nonlinear models, so it is necessary to introduce activation functions. We discuss two activation functions: sigmoid function $\sigma(x)=\frac{1}{1+e^{-x}}$ and tanh function $\sigma_2(x)=\frac{e^{x}-e^{-x}}{e^{x}+e^{-x}}$.

\paragraph{The SSL model.} We focus on a variant of  SimSiam \cite{DBLP:conf/cvpr/ChenH21}. SimSiam has shown impressive performance in various downstream experiments using only positive pairs and has become a representative SSL model. Fig. \ref{fig:SSL_model} shows the structure of the model in this paper. The datapoint $x_i$ is augmented by data augmentation $\xi_{\text{aug}}$ and $\xi'_{\text{aug}}$ to obtain two augmented datapoints $x'_i$ and $x''_i$. The data representations $z'_i$ and $z''_i$ of $x'_i$ and $x''_i$ are obtained through the nonlinear encoder $\sigma(Wx'_i)$ and $\sigma(Wx''_i)$. We use inner product $\langle z_i',z_i''\rangle$ to measure the similarity between $z'_i$ and $z''_i$. \citet{tian2021understanding} showed that a regularizer is essential for the existence of 
the non-collapsed solution. Hence, $\alpha \|W\|_F^2$ is added in $L$.
The objective function $L$ is defined as
\begin{align}
    &\min_{W}  L=\alpha\left\|W\right\|_F^2\notag\\ &-\frac{1}{n}\sum_{i=1}^n\mathbb{E}_{\xi_{\text{aug}},\xi'_{\text{aug}}}\Big[\left\langle \sigma (W(x_i+\xi_{\text{aug}})),\sigma(W(x_i+\xi'_{\text{aug}}))\right\rangle\Big]\,,\label{Objective_all_noise_empirical_pro}
\end{align}
where $\alpha$ is the coefficient of regularizer, $W=\left[w_1,w_2\right]^{\top}\in \mathbb{R}^{2\times d}$ and $\xi_{\text{aug}},\xi'_{\text{aug}} \sim \mathcal{N}\left(0,\rho^2 I\right)$. $W$ is the weight matrix of the encoder containing two neurons, and the parameters of the encoder are the same on both sides. 

Note that when \citet{liu2021self} took expectation over $\xi_{\text{aug}}, \xi'_{\text{aug}}$, they canceled the effect of $\xi_{\text{aug}}, \xi'_{\text{aug}}$ due to its linear framework. In other words, the variance of $\xi_{\text{aug}}, \xi'_{\text{aug}}$ can be arbitrarily large. However, \citet{jing2021understanding} showed that strong augmentation causes dimensional collapse. Hence it is necessary to consider the variance of the data augmentation. In our formulation, $\xi_{\text{aug}}$ and $\xi'_{\text{aug}}$ can not be canceled due to the nonlinear model. Thus our setting is more reasonable and more in line with the models in practice.
To deal with data augmentation operation, we adopt $\xi_{\text{aug}},\xi'_{\text{aug}} \sim \mathcal{N}\left(0,\rho^2 I\right)$. 
\begin{figure}[t]
\setlength{\abovecaptionskip}{10pt} 
    \centering
    \includegraphics[width=7.5cm,height=5.5cm]{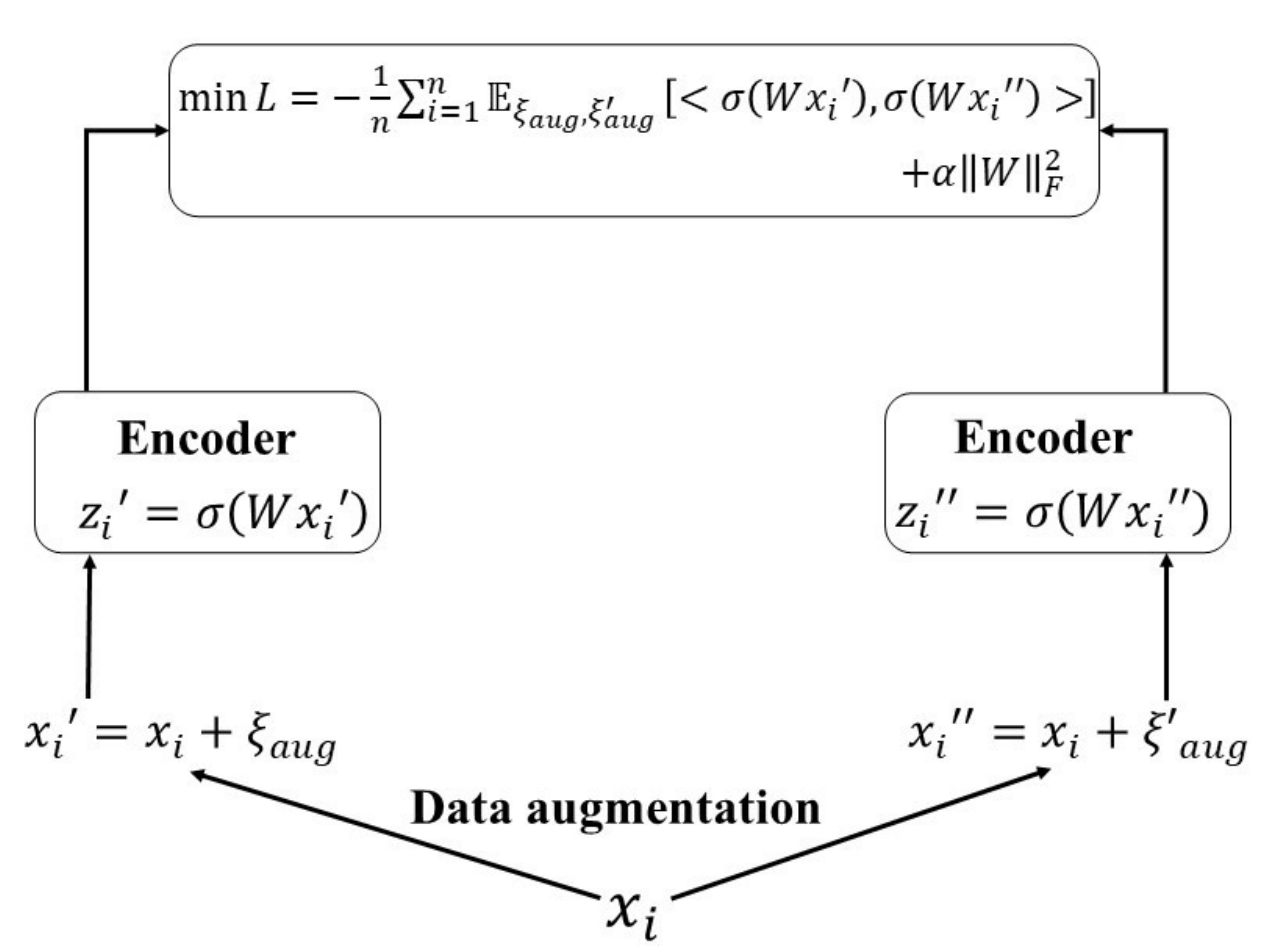}
    \caption{
    The structure of SSL model.
    }
    \label{fig:SSL_model}

\end{figure}

\paragraph{The SL model.}
We consider a simple two-layer nonlinear SL model to deal with the above 2-classification problem. 

Define $f_{F,W^{\text{SL}}}(x) \triangleq F\sigma\left(W^{\text{SL}}x\right)$ with $F\in \mathbb{R}^{1\times2}$ as the projection matrix and $W^{\text{SL}} \triangleq [w_1^{\text{SL}},w_2^{\text{SL}}]^{\top}\in\mathbb{R}^{2\times d}$ as the weight matrix of the feature extractor.
The usual process is to use sigmoid function to transform $f_{F,W^{\text{SL}}}(x_i)$ to $\widehat{y}_i\in(0,1), \forall i\in[n]$. Then, a binary cross-entropy loss function can be constructed with $\widehat{y}_i$ and label information $y_i, \forall i\in[n]$. 

However, in this paper, we focus on the performance of the feature extractor $\sigma\left(W^{\text{SL}}x\right)$. Therefore an objective function that minimizes the norm of feature extractor matrix $W^{\text{SL}}$ with margin constraint is used:
\begin{align}
        &\min_{W^{\text{SL}}} L_{\text{SL}} = \left\|w_1^{\text{SL}}\right\|_2^2+\left\|w_2^{\text{SL}}\right\|_2^2, \notag\\
    & \text{s.t.} \, \sigma\left((w_{y+1}^{\text{SL}})^\top x\right)-\sigma\left((w_{y'+1}^{\text{SL}})^\top x\right)\notag\\&\quad\ge \sigma(2)-\sigma(-2)-5\rho d^{\frac{1}{10}},\forall (x,y) \in \mathcal{D}^{\text{SL}}, y\neq y'\,.\label{objective_SL_ep}
\end{align}

The SL objective function in this paper is similar to the linear SL model in \citet{liu2021self}. We set this SL objective function mainly out of intuition: If a supervised learning model is good enough to complete the classification task, it should satisfy the above margin constraint.

\paragraph{Definitions and notations.}To characterize the objective functions, we give the following definitions and notations.

\begin{definition}[Locally strong convexity and smooth on $B_0$]\label{def:locally_strong_convex_and_L_smooth}
Function $f: \mathbb{R}^d \rightarrow \mathbb{R}$ is locally $\mu$-strongly convex and $L_m$-smooth if 
\begin{align}
    \mu I\preceq \nabla^2f(x)\preceq L_mI,\,\forall x\in B_0,
\end{align}
where $B_0:=\{x:\|x-x^*\|_2\leq\|x^{(0)}-x^*\|_2\}$ and $x^*\in \text{argmin}_{x\in \mathcal{X}}f(x)$.
\end{definition}
\begin{definition}[$L_H$-Lipschitz continuous Hessian]\label{def:Lipschitz_continuous_Hessian}
Function $f: \mathbb{R}^d \rightarrow \mathbb{R}$ is $L_H$-Lipschitz continuous Hessian if 
\begin{align}
    \left\|\nabla^2f(x)-\nabla^2f(y)\right\|_2\leq L_H\|x-y\|_2,\, \forall x,y\in \mathbb{R}^d.
\end{align}
\end{definition}

\paragraph{Notations.}\label{para:notations}
For $x\in \mathbb{R}^d$, we denote by $\|x\|_2$ the vector's Euclidean norm. For $A\in \mathbb{R}^{d\times d}$, we denote by $\|A\|_F$ the standard Frobenius norm and define $\|A\|_2=\sqrt{\lambda}$ where $\lambda$ is the largest eigenvalue of $A^{\top}A$. For $x\in \mathbb{R}^d$ and $\nabla^3 f(x)\in \mathbb{R}^{d\times d\times d}$, we give an upper bound of $\|\nabla^3 f(x)\|_2$ by considering $\nabla^3 f(x)$ as a matrix-vector. Each element of the matrix-vector is $\frac{\partial \nabla^2f(x)}{\partial x_i} \in \mathbb{R}^{d\times d}.$ It is clear that $\left\|\nabla^3 f(x)\right\|_2^2\leq \sum_{i=1}^d \left\|\frac{\partial \nabla^2f(x)}{\partial x_i}\right\|_F^2$. We denote by $O(\cdot)$ standard Big-O notations, only hiding constants. We denote by $z^{(k)}$ the $k$-th element of $z\in \mathbb{R}^d$ and $z(t)$ the $t$-th iteration of the gradient descent algorithm.

\section{SSL is Superior to SL in Learning Representation}\label{sec:proofs}

In this section, we show that the nonlinear SSL model can capture the label-related feature and the hidden feature of data distribution at the same time. In contrast, the nonlinear SL model can only learn the label-related feature.
For simplicity, we assume $e_1 = (1,0,...,0)^\top, e_2 = (0,1,...,0)^\top\in\mathbb{R}^d$.

\subsection{The Learning Abilities of SSL and SL}
For the convenience, we define
$D_1(\tau)=\{\vec{x}\in \mathbb{R}^d|x^{(1)}\in(3.1,3.9), \tau x^{(2)}\in(8.5,9), x^{(k)} \in(-\frac{3}{d^{0.49}},\frac{3}{d^{0.49}}), \forall k\in [3,d]\}$ and $D_2(\tau)=\{\vec{x}\in \mathbb{R}^d|x^{(1)}\in(-3.9,-3.1), \tau x^{(2)}\in(8.5,9), x^{(k)} \in(-\frac{3}{d^{0.49}},\frac{3}{d^{0.49}}), \forall k\in [3,d]\}$ as the initialization region of $w_1$ and $w_2$ in Theorem \ref{thm:Convergence_noiseepexp_pro}. 

\begin{restatable}{theorem}{Convergencenoiseepexppro}\label{thm:Convergence_noiseepexp_pro}
For $\alpha = 1/800 ,\tau = \max \{7, d^\frac{1}{10}\}, \rho=1/d^{1.5}$ and $n=d^2$, with probability $1-O\left(e^{-d^\frac{1}{10}}\right)$, the SSL objective function $L$ exists a local minimum $W^* = (w_1^{*}, w_2^*)^{\top}:$
\begin{align*}
    \|w_1^*-\widetilde{w}_1^{*}\|_2\leq O(d^{-\frac{1}{2}})\,,\\
    \|w_2^*-\widetilde{w}_2^{*}\|_2\leq O(d^{-\frac{1}{2}})\,,
\end{align*}
where $\widetilde{w}_{1}^{*(1)}\in[3.1,3.9]$, $\widetilde{w}_{1}^{*(1)}=-\widetilde{w}_{2}^{*(1)}, \tau \widetilde{w}_{1}^{*(2)} = \tau \widetilde{w}_{2}^{*(2)}\ge 9, \widetilde{w}_{1}^{*(k)}=\widetilde{w}_{2}^{*(k)}=0, \forall k\in[3,d]$.

Furthermore, when $\big(w_1(0),w_2(0)\big)\in D_1(\tau)\times D_2(\tau)$, using the gradient descent algorithm and choosing learning rate $\eta =\frac{2}{4\alpha +\tau^2+1.5}, \kappa = 1+ \frac{\tau^2+1.5+2d^{-0.1}}{2\alpha-d^{-0.1}},$ we have
\begin{align*}
    \|w_1(t)-w_1^{*}\|_2\leq \left(\frac{\kappa-1}{\kappa+1}\right)^t\|w_1(0)-w_1^{*}\|_2\,,\\
    \|w_2(t)-w_2^{*}\|_2\leq \left(\frac{\kappa-1}{\kappa+1}\right)^t\|w_2(0)-w_2^{*}\|_2\,.
\end{align*}

The projection of $e_1$ and $e_2$ on the space spanned by $w_1^{*}$ and $w_2^{*}$ is very close to $1,$ i.e.,
\begin{align*}
    |\Pi e_1|\ge 1-O(\tau^3d^{-\frac{1}{2}})\,,\\
    |\Pi e_2|\ge 1-O(\tau^3d^{-\frac{1}{2}})\,.
\end{align*}
\end{restatable}
Theorem \ref{thm:Convergence_noiseepexp_pro} shows that using GD to train the nonlinear SSL model under a certain initialization region $D_1(\tau)\times D_2(\tau)$, the model can converge to a local minimum $(w_1^*, w_2^*)$. Further, the projection of $e_1,e_2$ on the space spanned by $w_1^{*}, w_2^{*}$ is almost $1$. In other words, the nonlinear SSL model has simultaneously learned $e_1$ and $e_2$, which are the label-related and hidden features.

\begin{restatable}{theorem}{SLtheoremep}\label{thm:SLtheoremep}
Let $w_1^{\text{SL},*}$ and $w_2^{\text{SL},*}$ be the optimal solution of $L_{\text{SL}}$. Then with probability $1-O(e^{-d^\frac{1}{10}}),$
\begin{align}
    \left(w_{1}^{\text{SL},*(2)}\right)^2+\left(w_{2}^{\text{SL},*(2)}\right)^2\leq O\left(\rho d^{\frac{1}{10}}\right)\,.\notag
\end{align}
When $\rho = 1/d^{1.5}, \left(w_{1}^{\text{SL},*(2)}\right)^2+\left(w_{2}^{\text{SL},*(2)}\right)^2\leq O\left(1/d^{1.4}\right)$.
\end{restatable}
From Theorem \ref{thm:SLtheoremep}, we show that $ (w_{1}^{\text{SL},*(2)})^2+(w_{2}^{\text{SL},*(2)})^2$ is very small, which means SL model can only learn label-related feature and some noise terms. 

Note that the previous works \cite{tian2017analytical,zhang2019learning,li2018learning} used the gradient-based algorithm to analyze the SL model with one hidden layer and obtained asymptotic convergence guarantees. They did not analyze the specific form of the learning results. Hence, Theorem \ref{thm:SLtheoremep} is different compared with the previous results. We obtain the bounds of each dimension of the learning results by constructing margin constraints. These bounds accurately describe the features learned by the SL model and help to characterize the representation learnability of the SL model. 

Finally, Theorem \ref{thm:Convergence_noiseepexp_pro} and Theorem \ref{thm:SLtheoremep} show that the nonlinear SSL model is superior to the nonlinear SL model in capturing important data features, which means SSL can obtain a more competitive data representation than SL.

\subsection{Discussion}\label{subsec:Hardness}
\paragraph{The extension to more general data distributions.} As described in Sec 3.1, we treat label-related features as a group, represented by $e_1$ ($e_2$ represents hidden features), and obtain Theorem 1. In this part, we show that Theorem 1 can be extended to data distributions with many label-related and hidden features. Suppose there are $P$ label-related features $E^L=\{e_1, \ldots,e_P\}$ and $Q+1$ hidden features $E^H=\{e_{P+1},\ldots, e_{P+Q},\vec{0}\}$, where $E=\{E^L,E^H\}$ is column-orthogonal matrix. Each datapoint consists of a label-related feature and a hidden feature, $x_i =z_i e_i^L+\tau e_i^H$, where $P(z_i=1)=P(z_i=-1)=1/2$. $e_i^L$ and $e_i^H$ are features in $E^L$ and $E^H$. This general distribution only considers the relationship between label-related features and hidden features. Hence, the gradient can be decoupled, and the method of this paper can be 
applied. Finally, we can know that if $W$ contains $P+Q$ neurons, the learning results of $W$ will span the space spanned by $\{e_1,\ldots,e_{P+Q}\}$. This conclusion can be regarded as the general version of Theorem 1.
\paragraph{The challenges for nonlinear models.}
Since the objective function is non-convex and nonlinear, it is difficult to get a closed-form solution with a similar process of \citet{liu2021self}. We need to use an optimization algorithm such as GD to converge to a local minimum $(w_1^*, w_2^*)$ and determine features captured by $(w_1^*, w_2^*).$

For nonlinear SSL models, previous work \cite{wen2021toward} used SGD to update the model step-by-step and observed the learning result during the iteration. However, the step-by-step process is complex, and it is easy to ignore the change process. Therefore, this procedure makes it difficult to analyze the learning results of local minima accurately.

Different from the complex iterative analysis of the previous work, we propose a new analysis process. We first obtain the approximate region and properties of the local minimum from the simplified objective function $\widetilde{L}$ and then extend it to the original complex objective function $L$. For the transformation from $\widetilde{L}$ to $L$, we use the \textbf{exact version of Inverse Function Theorem} as a bridge, avoiding the direct analysis of the local minimum of $L$. 
In the remainder of this part, we demonstrate the intuitions and techniques for each part.

\textbf{Non-convex and nonlinear objective function.} At this step, we consider the structure of the objective function, ignore noise terms, and take expectation over data distribution:
\begin{align}
\min \widetilde{L}=-\mathbb{E}_{\widetilde{x}}[\langle \sigma(W\widetilde{x}),\sigma(W\widetilde{x})\rangle]+\alpha\|W\|_F^2\,, \notag
\end{align}
where $\widetilde{x}_i$ is the datapoint without noise term $\rho \xi_i$. We use the intermediate value principle, locally strong convexity of $\widetilde{L}$, and the properties of activation function carefully to prove the existence of the local minimum $(\widetilde{w}_1^*, \widetilde{w}_2^*)$ of $\widetilde{L}$.

\textbf{The exact version of Inverse Function Theorem.} There are many noise terms in $L$, such as $\rho\xi_i, \forall i\in[n]$ (datapoint noise), $\xi_{\text{aug}}$ (data augmentation noise), and the error terms due to the expectation operation over the data distribution. After obtaining the upper bound of these noise terms (Sec. \ref{subsec:proof_sketch}), we need a bridge to deal with the transformation from $\widetilde{L}$ to $L$. Since $(\widetilde{w}_1^*, \widetilde{w}_2^*)$ is local minimum of $\widetilde{L}$ and noise terms are bounded, $L$ should be $\mu$-strongly convex and $L_m$-smooth in the neighborhood of $(\widetilde{w}_1^*, \widetilde{w}_2^*)$. With these properties, it is clear that $\frac{\partial L}{\partial w_1}$ is one-to-one in a small neighborhood of $\widetilde{w}_1^*$ using the origin Inverse Function Theorem \cite{rudin1976principles}. However, we need a exact neighborhood to guarantee that the solution $w_1^*$ of $\frac{\partial L}{\partial w_1}=0$ is in the one-to-one region. Hence we introduce Lipschitz continuous Hessian constant $L_H$ to build an open ball centered at $w_1^*$ with radius $r=\frac{1}{2\mu L_H}$ as the exact neighborhood and modify the Inverse Function Theorem to complete our proof.
\begin{figure}[t]
\setlength{\abovecaptionskip}{10pt} 
    \centering
    \includegraphics[width=7cm,height=4cm]{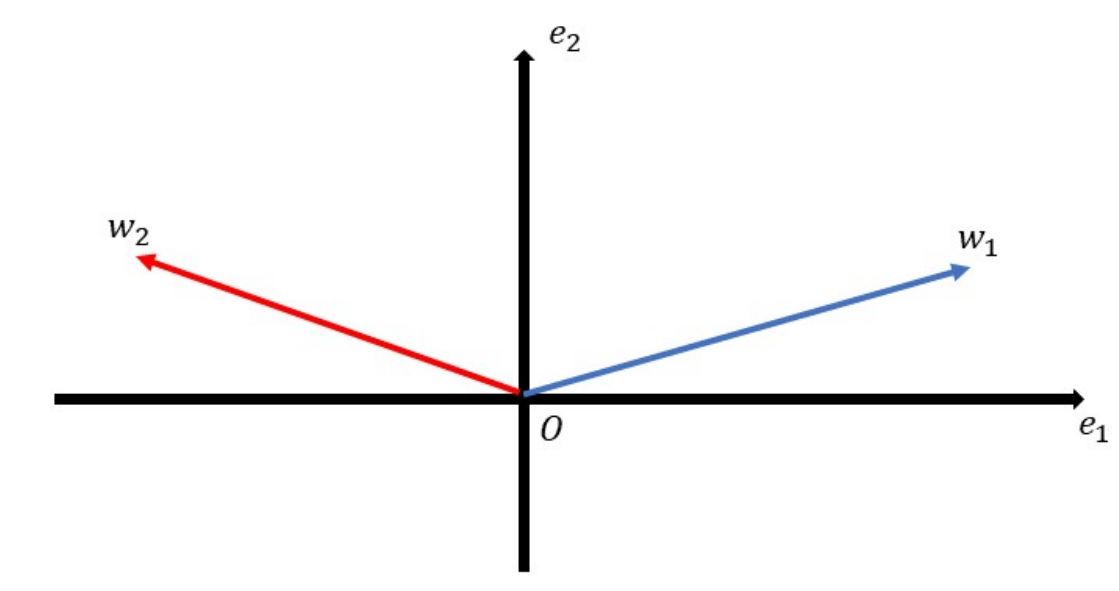}
    \caption{
    Theoretical Results of Theorem \ref{thm:Convergence_noiseepexp_pro}
    }
    \label{fig:Theoretical_Results}
\end{figure}
\begin{figure*}[t]
\setlength{\abovecaptionskip}{10pt} 
    \begin{minipage}[t]{0.32\linewidth}
        \centering
        \includegraphics[width=0.98\textwidth]{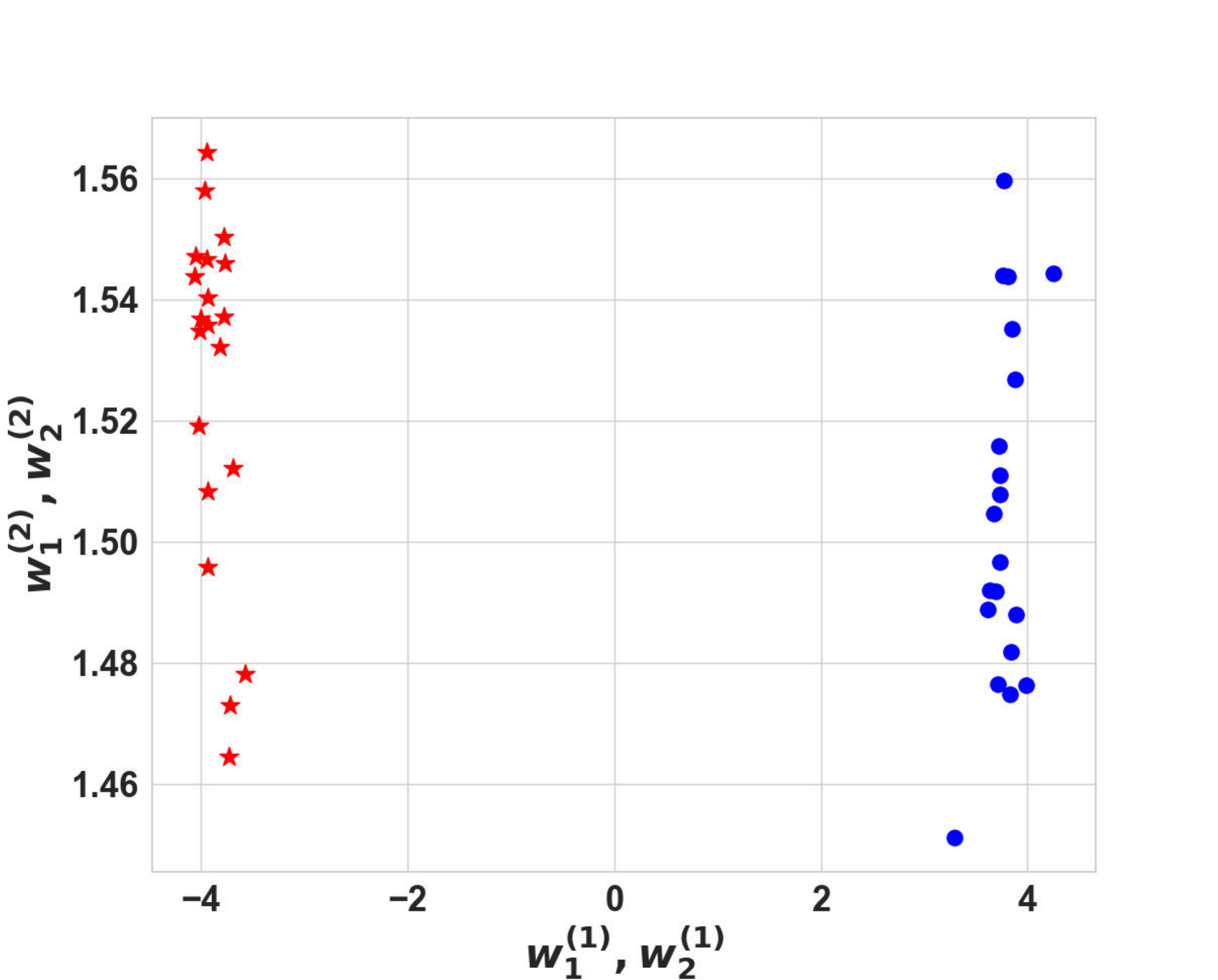}
        \subcaption{Final weight matrix $W$}
        \label{fig:d10tau7final_matrix}
    \end{minipage}
    \begin{minipage}[t]{0.32\linewidth}
        \centering
        \includegraphics[width=0.98\textwidth]{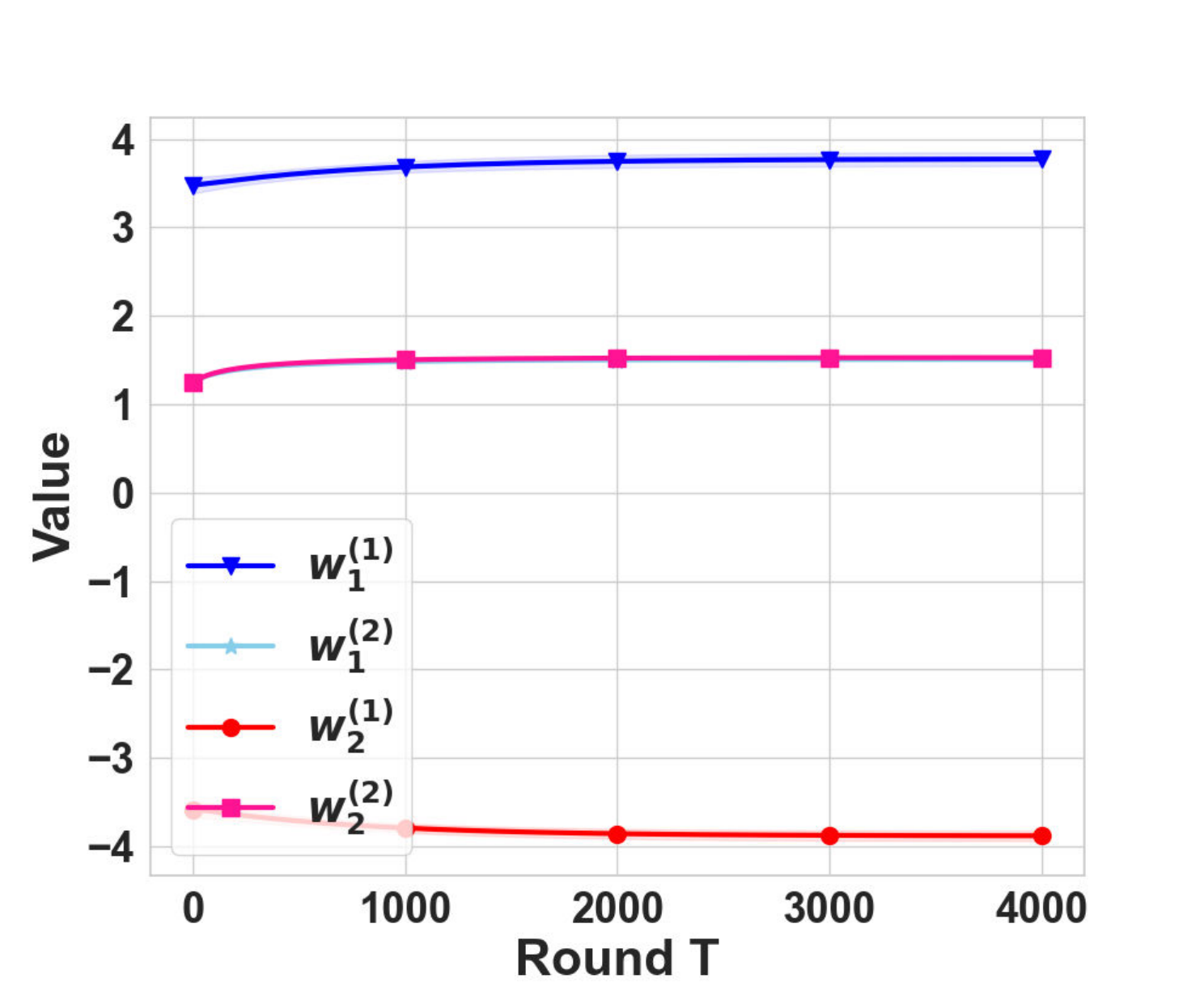}
        \subcaption{Learning curve}
        \label{fig:d10tau7learning_curve}
    \end{minipage}
    \begin{minipage}[t]{0.32\linewidth}
        \centering
        \includegraphics[width=0.98\textwidth]{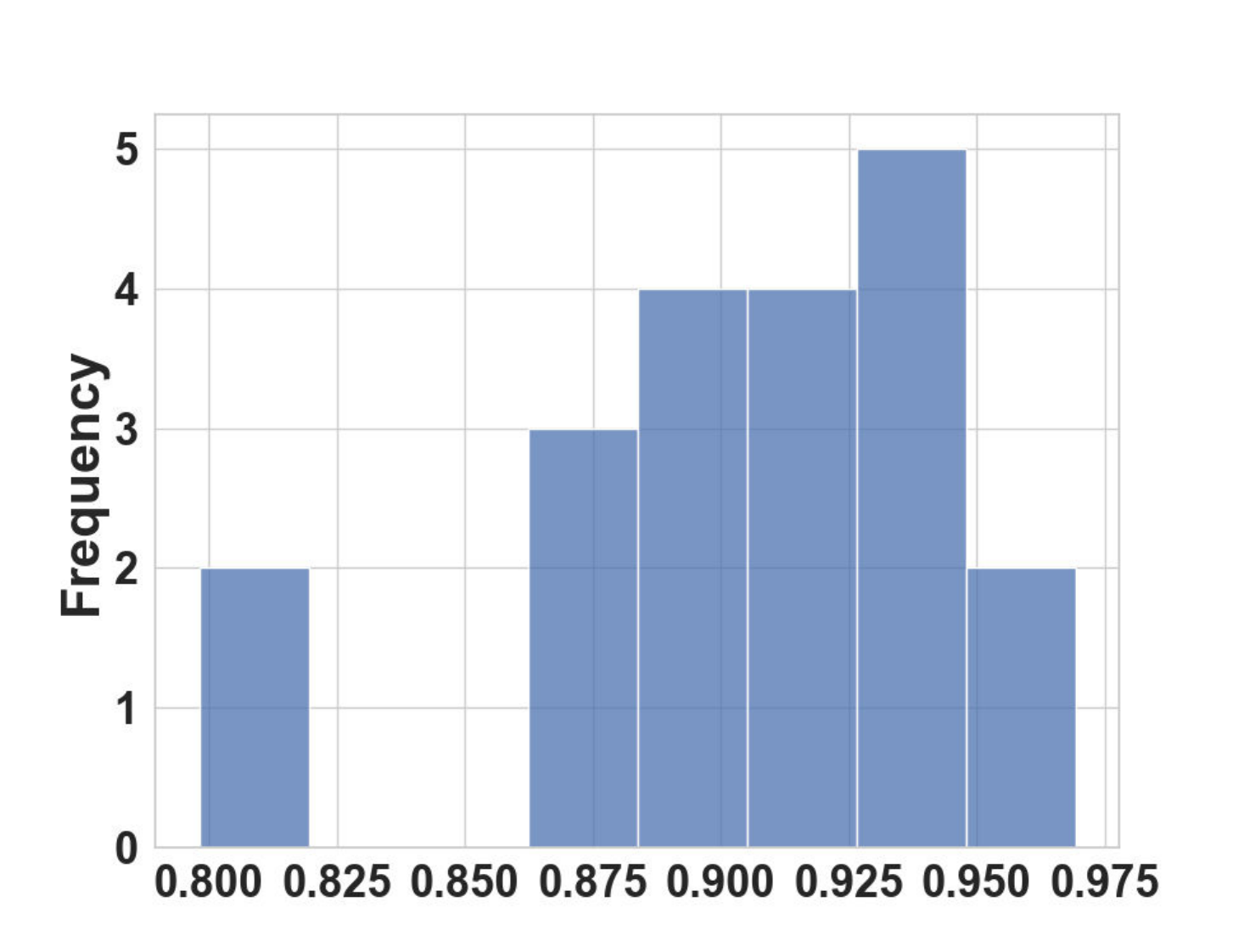}
        \subcaption{The projection of $e_2$}
        \label{fig:d10tau7projection_e2}
    \end{minipage}
    \caption{Experiment results of SSL model with $d=10, \tau=7$}
    \label{fig:d10tau7p2}
\end{figure*}
\subsection{Proof Sketch of Main Theorem}\label{subsec:proof_sketch}

{\textbf{Proof sketch of SSL}.} For the sake of discussion, we respectively define $\widetilde{D}_1(\tau) = \{\vec{x}\in \mathbb{R}^d|x^{(1)}\in[3.1,3.9], \tau x^{(2)}\in[9,+\infty), x^{(k)}=0, \forall k\in[3,d]\}$ and $\widetilde{D}_2(\tau) = \{\vec{x}\in \mathbb{R}^d|x^{(1)}\in[-3.9,-3.1], \tau x^{(2)}\in[9,+\infty), x^{(k)}=0, \forall k\in[3,d]\}$ as the region of $\widetilde{w}_1^{*}$ and $\widetilde{w}_2^{*}$. 

As a beginning, we focus on $\widetilde{L}$. To obtain the solution of $\frac{\partial \widetilde{L}}{\partial w_1}=0$, we first solve $\frac{\partial \widetilde{L}}{\partial w_{1}^{(k)}}=0, \forall k\in [2]$ separately in $\widetilde{D}_1(\tau)$. Subsequently, we use the intermediate value principle twice to prove the existence of $\widetilde{w}_1^{*}$. 
Finally, we use the Hessian matrix to prove that $\widetilde{w}_1^{*}$ is a local minimum. We demonstrate that  $\widetilde{L}$ is $\widetilde{\mu}$-strongly convexity and $\widetilde{L}_{m}$-smooth in the region around $\widetilde{w}_1^{*}$.

To prove Theorem \ref{thm:Convergence_noiseepexp_pro}, we need to deal with the noise terms in $L$.  Due to the activation function, we cannot use the noise matrix to treat the noise terms as in \citet{liu2021self}. Hence, we use the Lagrange’s Mean Value Theorem to separate $\xi_i, \xi_{\text{aug}}, \xi^{\prime}_{\text{aug}}$ from the activation function and bound these noise terms using the tail bound of Gaussian variable. There are also some error terms due to the expectation operation over data distribution. With the intuition that $n_l$, $\forall l\in[4]$ can not be far away from $n/4$, we bound these error terms.

After obtaining the upper bound of the above noise terms, we characterize the landscape of $L$ by using the Matrix Eigenvalue
Perturbation Theory~\cite{kahan1975spectra}. We sum up the properties of $L$ when $w_1$ around $\widetilde{w}_1^*$ as follows. 
\begin{enumerate}
    \item $\frac{\partial L}{\partial w_1}|_{w_1=\widetilde{w}_1^{*}}$ is very close to $0$.
    \item $L$ is $\mu$-strongly convex and $L_m$-smooth. Specifically, we show that $\widetilde{\mu} -\epsilon_1\leq\mu\leq \widetilde{\mu}$ and $\widetilde{L}_m \leq L_m\leq  \widetilde{L}_m+\epsilon_1$ where $\epsilon_1$ is a small term related to $\rho$ and $d$.
    \item $L$ is $L_H$-Lipschitz continuous Hessian.
\end{enumerate} 

Combined with these properties, we use the exact version of Inverse Function Theorem to prove the existence of 
the local minimum $(w_1^*, w_2^*)$ of $L$. Finally, we show that with good initialization, specifically initialization around the local minimum, $w_1(0)$ converges to $w_1^*$ using the gradient descent algorithm \cite{bubeck2015convex}.
We remark that the above process only analyzes $w_1$, we can get $w_2$ through symmetry.
\begin{figure}[t] 
\setlength{\abovecaptionskip}{10pt} 
\centering
\includegraphics[width=0.7\linewidth]{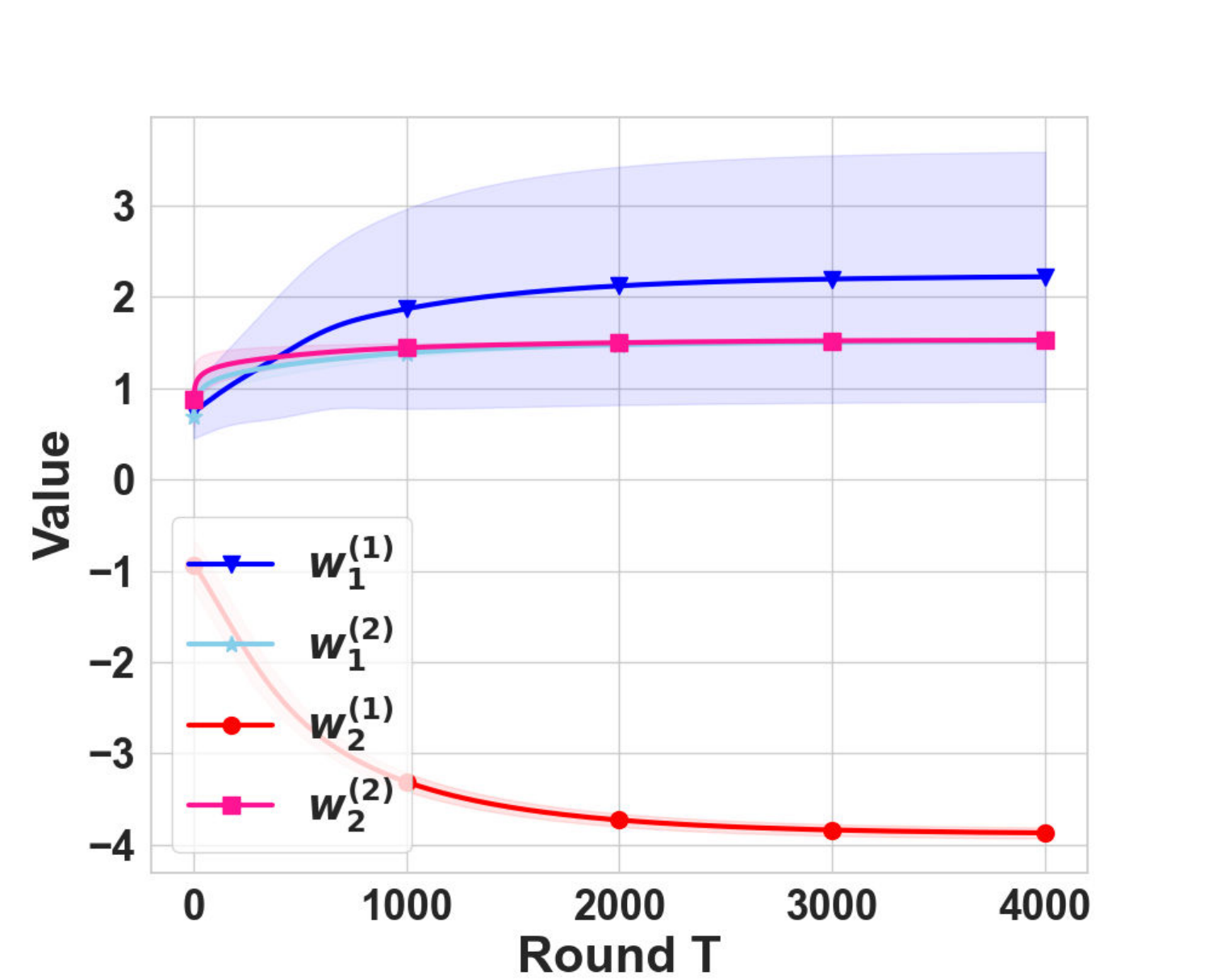}\\
\includegraphics[width=0.7\linewidth]{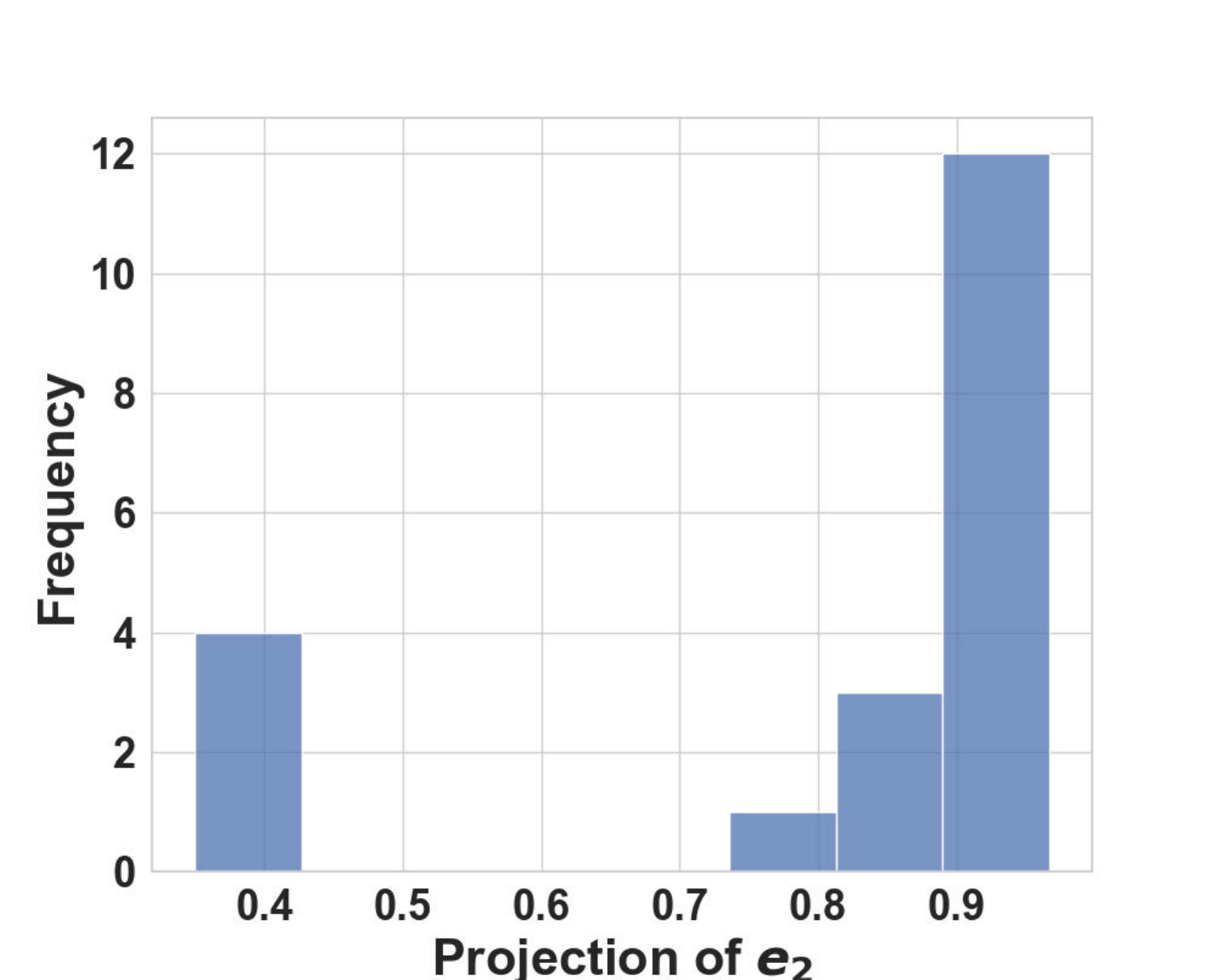}
  \caption{The experiments results with the correct sign}
  \label{fig:correct_sign}
\end{figure}
\paragraph{Proof sketch of SL.} The proof sketch of SL is similar to the proof of the linear SL model in \citet{liu2021self}. However, because of the nonlinear SL model in this paper, we need to perform finer scaling to get a high probability guarantee.

\paragraph{Different activation function.} We can easily extend the results to the case where the activation function is tanh because sigmoid can be viewed as a compressed version of tanh.

To get similar results with Theorem \ref{thm:Convergence_noiseepexp_pro}, we just need to modify the region of the local minimum and the initialization region. For $\widetilde{D}_1(\tau)$, we change the range of $x^{(1)}$ from $[3.1,3.9]$ to $[2.7,3.1]$ and the range of $x^{(2)}$ from $[9,+\infty)$ to $[6.1,+\infty)$ to obtain $\widetilde{D}_1^{\sigma_2}(\tau)$. For $D_1(\tau)$, we change the range of $x^{(1)}$ from $(3.1,3.9)$ to $(2.7,3.1)$ and the range of $x^{(2)}$ from $(8.5,9)$ to $(5.75,6.1)$ to obtain $D_1^{\sigma_2}(\tau)$. With similar process, we can get $\widetilde{D}_2^{\sigma_2}(\tau)$ and $D_2^{\sigma_2}(\tau)$. 

\begin{figure}[t]
\setlength{\abovecaptionskip}{10pt} 
    \centering
    \includegraphics[width=6cm,height=4.6cm]{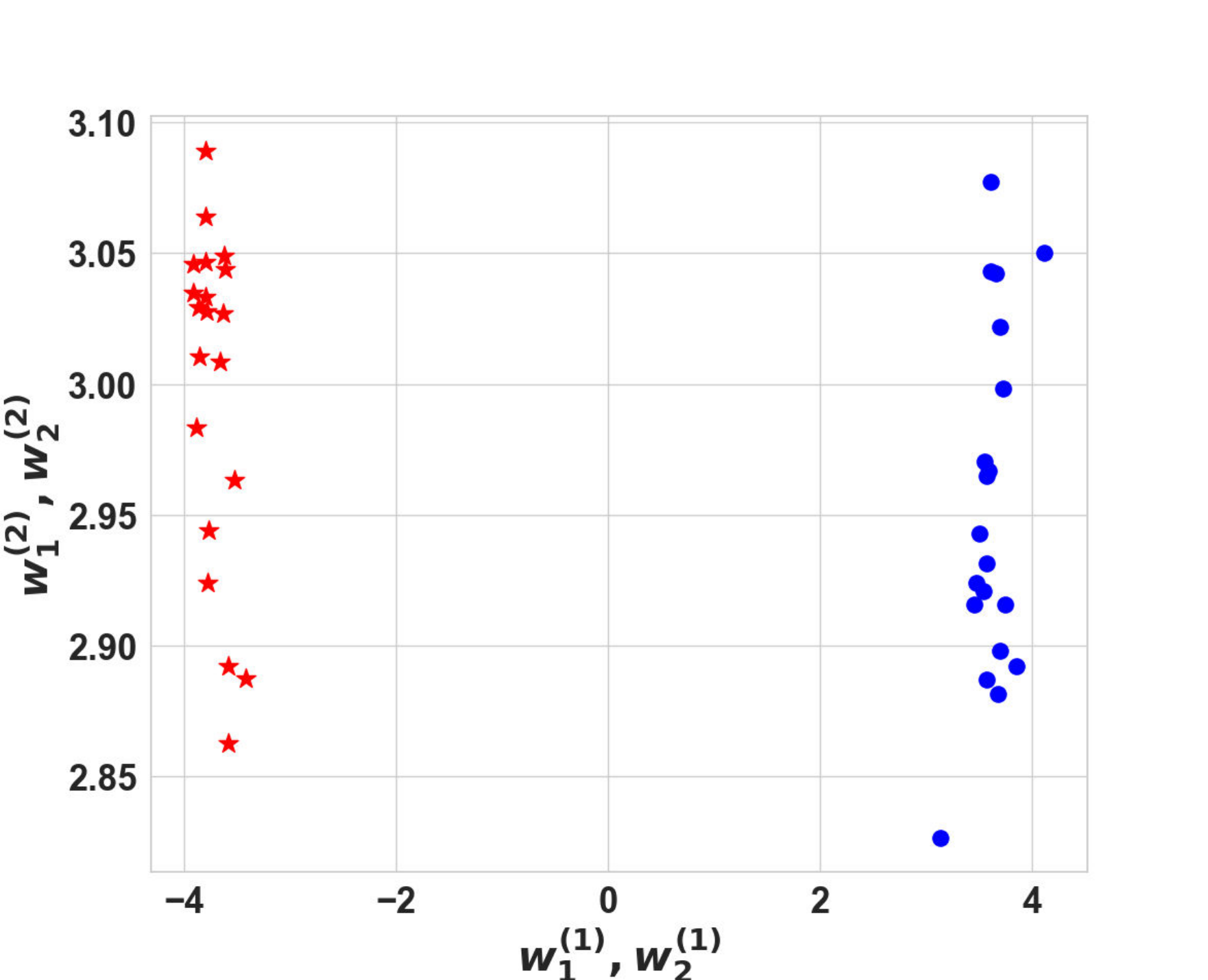}
    \caption{
SSL final weight matrix $W$ with $d=10, \tau=3$
    }
    \label{fig:d10tau3}
\end{figure}
\begin{figure*}[t]
\setlength{\abovecaptionskip}{10pt} 
    \begin{minipage}[t]{0.32\linewidth}
        \centering
        \includegraphics[width=0.98\textwidth]{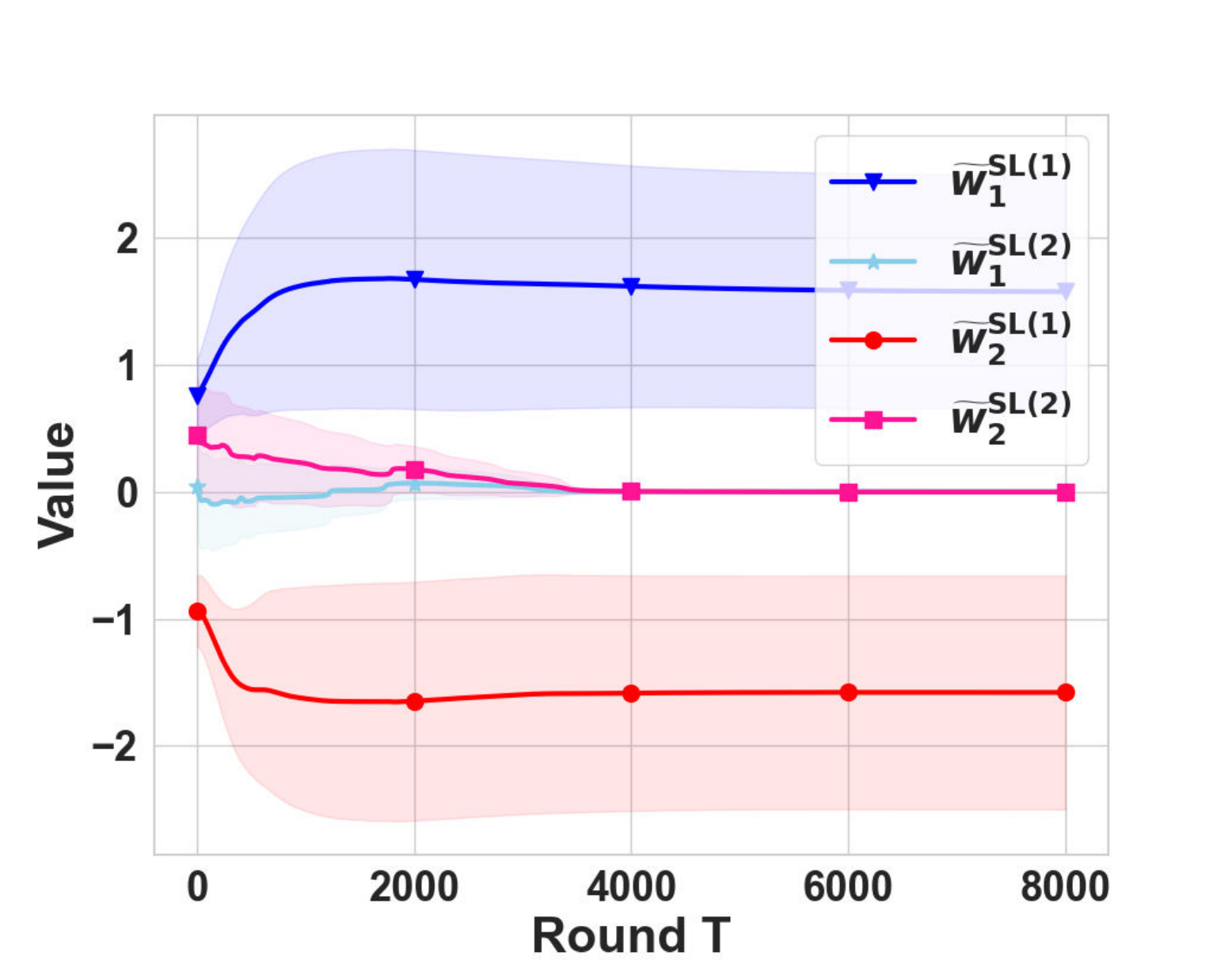}
        \subcaption{Learning curve}
        \label{fig:SLlearning_curve}
    \end{minipage}
    \begin{minipage}[t]{0.32\linewidth}
        \centering
        \includegraphics[width=0.98\textwidth]{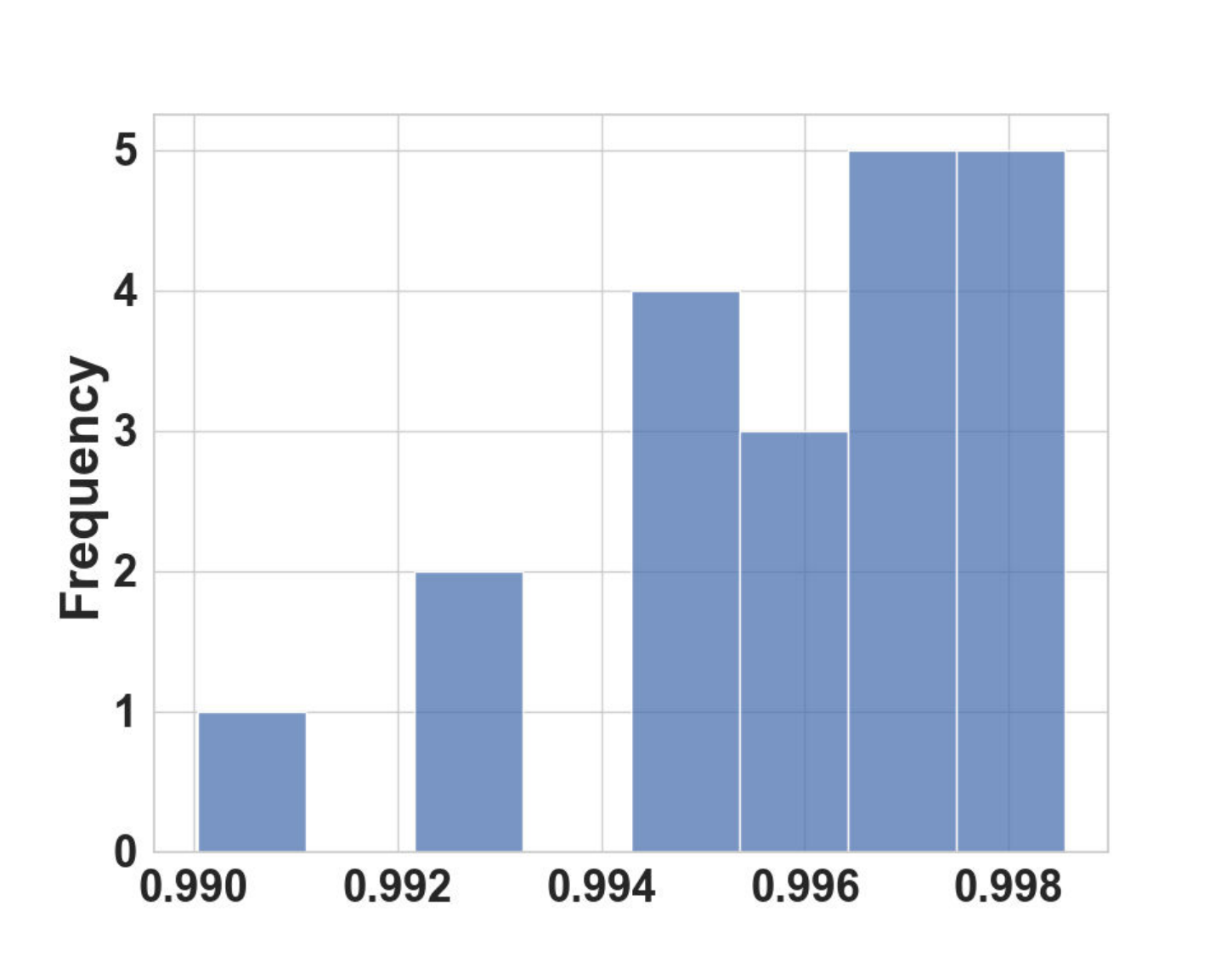}
        \subcaption{The projection of $e_1$}
        \label{fig:SLprojection_e1}
    \end{minipage}
    \begin{minipage}[t]{0.32\linewidth}
        \centering
        \includegraphics[width=0.98\textwidth]{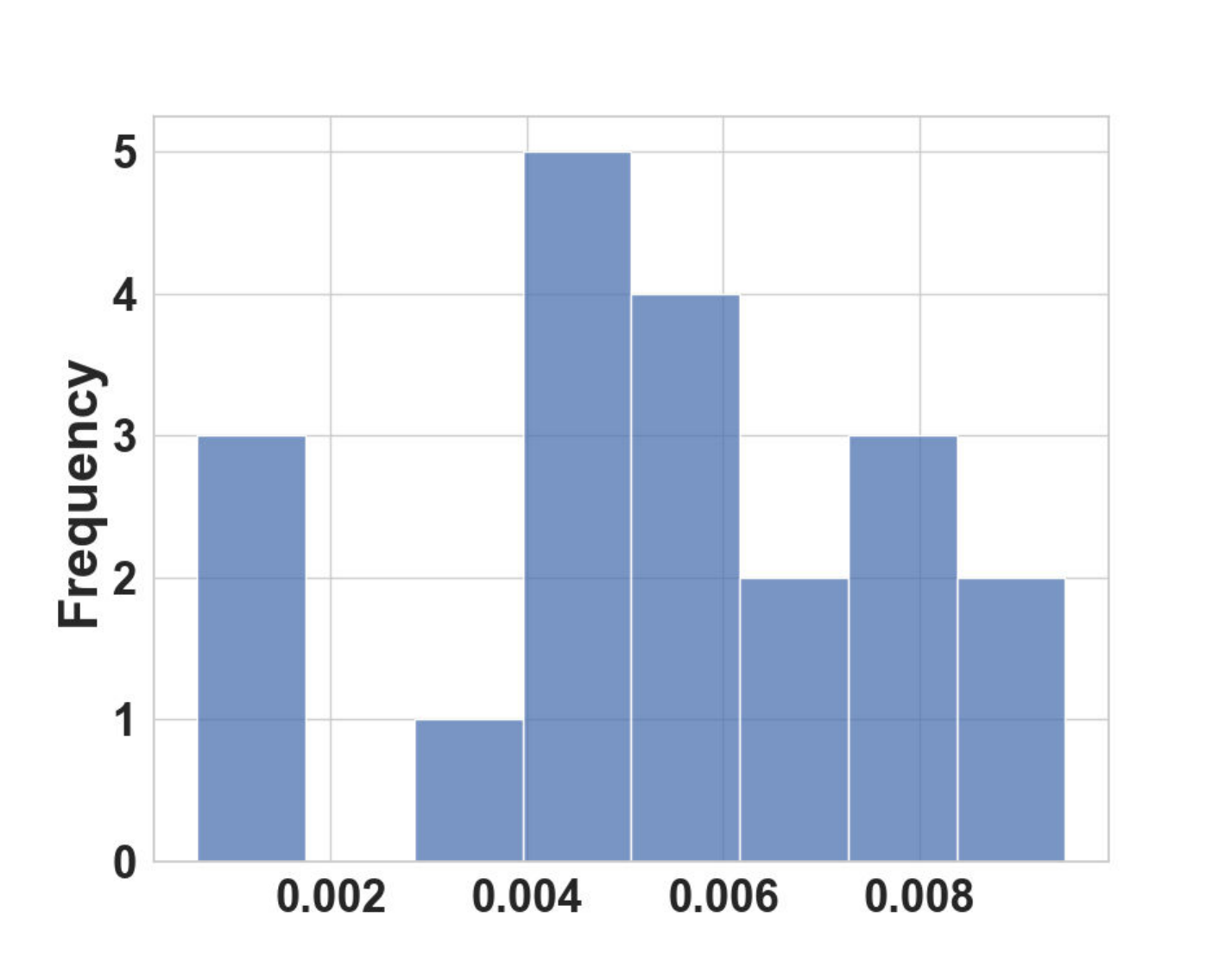}
        \subcaption{The projection of $e_2$}
        \label{fig:SLprojection_e2}
    \end{minipage}
    \caption{Experiment results of SL model with $d=10, \tau=7$}
    \label{fig:SLresult}
\end{figure*}
\section{Simulation Experiments}\label{exp}
In this section, we illustrate the correctness of Theorem \ref{thm:Convergence_noiseepexp_pro} and Theorem \ref{thm:SLtheoremep} experimentally. We conduct experiments for the nonlinear SSL model in Sec. \ref{subsec:exp_result} and Sec. \ref{subsec:beyond_analysis_exp}. Furthermore, we show the training process of the nonlinear SL model with projection matrix $F$ in Sec. \ref{subsec:SL_exp_result}. In this section, we choose $\tau =7, d=10, \rho = 1/d^{1.5}, \alpha = \frac{1}{800}, n=d^2$ and learning rate $\eta= 0.001$ if we do not specify otherwise.  Experiments are averaged over 20 random seeds, and we show the average results with $95\%$ confidence interval for learning curves. 

\subsection{SSL Model: the Correctness of Theorem \ref{thm:Convergence_noiseepexp_pro}}\label{subsec:exp_result}
In this part, we validate the correctness of Theorem \ref{thm:Convergence_noiseepexp_pro} by strictly following the settings of the  theorem. Define $T_1=4000$ as the number of iterations of the SSL model. Fig. \ref{fig:d10tau7final_matrix} shows the learning result of weight matrix $W = [w_1, w_2]^\top$. The blue points are learning results of $(w_1^{(1)}(T_1),w_1^{(2)}(T_1))$ and the red stars are learning results of $w_2$. It is clear that $w_1(T_1)$ and $w_2(T_1)$ are almost symmetrical about the $e_2$-axis, which is consistent with the theoretical result (Fig. \ref{fig:Theoretical_Results}). Fig. \ref{fig:d10tau7learning_curve} shows the learning process of $w_1$ and $w_2$. Because we initialize $w_1(0)$ and $w_2(0)$ around the local minimum, $w_1$ and $w_2$ can easily converge to $(w_1^*, w_2^*)$. Fig. \ref{fig:d10tau7projection_e2} shows the projection of $e_2$ on the space spanned by $w_1(T_1)$ and $w_2(T_1)$. We can find the projection is almost $1$. These experimental results show that $W$ learns $e_1$, $e_2$ at the same time. The results of larger $\tau$ are similar to the results of $\tau =7$.

\subsection{SSL Model: Results Beyond Analysis}\label{subsec:beyond_analysis_exp}
In this part, we relax requirements in Theorem \ref{thm:Convergence_noiseepexp_pro}, such as $(w_1, w_2)$ must be initialized near $(w_1^*, w_2^*)$, $\tau$ must be large. We show that the SSL model still learns the label-related and hidden feature even if the requirements are relaxed.

\textbf{Good enough initialization.} In Theorem \ref{thm:Convergence_noiseepexp_pro}, we initialize $w_1$ and ${w_2}$ around $(w_1^*, w_2^*)$. We experimentally show that initialization only need the correct sign $(w_{1}^{(1)}(0)>0, w_{2}^{(1)}(0)<0, w_{1}^{(2)}(0)>0, w_{2}^{(2)}(0)>0)$ is required.
Fig. \ref{fig:correct_sign} shows that if the initialization sign is correct, the SSL model can converge to $(w_1^*, w_2^*)$ with high probability.
``With high probability" means there are still a few cases where the SSL model cannot converge to $(w_1^*, w_2^*)$. However, compared with the learning results (Fig. \ref{fig:SLprojection_e2}) of the SL model, the SSL model with the correct sign still shows the ability to learn $e_2$. 

\textbf{Large enough $\tau$.} In the proof process of Theorem \ref{thm:Convergence_noiseepexp_pro}, we need $\tau = \text{max} \{7, d^{\frac{1}{10}}\}$ to use the monotonicity of the solution of $\frac{\partial \widetilde{L}}{\partial w_{1}^{(2)}}$. We experimentally show that the SSL model can get a good result even if $\tau$ does not meet this requirement. Fig. \ref{fig:d10tau3} shows even if $\tau =3$ , the space spanned by $w_1$ and $w_2$ is still very close to the space spanned by $e_1$ and $e_2$.

\subsection{SL Experiment Results}\label{subsec:SL_exp_result}
In Theorem \ref{thm:SLtheoremep}, we mainly focus on the performance of the feature extractor $W^{\text{SL}}$ and ignore the projection matrix $F$. In this section, we experimentally show that even if $F$ is considered, $W^{\text{SL}}$ still only learns the label-related feature. Specifically, we consider the binary cross-entropy loss function:
\begin{align*}
    \min \widetilde{L}_{\text{SL}} = &-\frac{1}{n}\sum_{i=1}^n y_i \ln(\hat{y}_i)+(1-y_i)\ln (1-\hat{y}_i)\\&\qquad+\beta\|W^{\text{SL}}\|_F^2+\gamma\|F\|_2^2\,,
\end{align*}
where $\hat{y}_i = \sigma(F\sigma(W^{\text{SL}}x_i)), \forall i\in[n]$, $\beta$ is the coefficient of $W^{\text{SL}}$ regularizer, and $\gamma$ is the coefficient of $F$ regularizer. In this section, we choose $\beta=\gamma= 1/800$.

Define $T_2=8000$ as the number of iterations of the nonlinear SL model. Fig. \ref{fig:SLlearning_curve} shows the learning curve of $(\widetilde{w}^{\text{SL}}_1, \widetilde{w}^{\text{SL}}_2)$. It is clear that $\widetilde{w}_1^{\text{SL}(1)}(T_2)$ and $\widetilde{w}_2^{\text{SL}(1)}(T_2)$ are the main terms, and the other terms $\widetilde{w}_j^{\text{SL}(k)}, \forall j\in [2],k\in [2,d]$ will converge to $0$. Fig. \ref{fig:SLprojection_e1} and Fig. \ref{fig:SLprojection_e2}  show the projection of $e_1$ and $e_2$ on the space spanned by $\widetilde{w}_1^{\text{SL}}(T_2)$ and $\widetilde{w}_2^{\text{SL}}(T_2)$. It is clear that the projection of $e_1$ is almost $1$, and the projection of $e_2$ is almost $0$. The above experiment results mean that the nonlinear SL model can only learn label-related feature $e_1$, which is consistent with the results of Theorem \ref{thm:SLtheoremep}.

All experiments are conduct on a desktop with AMD Ryzen 7 5800H with Radeon Graphics 3.20 GHz and 16 GB memory. The codes of this section are available at  https://github.com/wanshuiyin/AAAI-2023-The-Learnability-of-Nonlinear-SSL.

\section{Conclusion}\label{sec:conclusion}
\textbf{Summary.} Our paper is the first to analyze the data representation learnability of the nonlinear SSL model by analyzing the learning results of the neural network. We start with a 1-layer nonlinear SSL model and use GD to train this model. We prove that the model converges to a local minimum. Further, we accurately describe the properties of this local minimum and  prove that the nonlinear SSL model can capture label-related features and hidden features at the same time. In contrast, the nonlinear SL model only learns label-related features. This conclusion shows that even though the nonlinear network significantly improves the learnability of the SL model, the SSL model still has a superior ability to capture important features compared with the SL model. We verify the correctness of the results through
simulation experiments.

Due to the nonconvexity of the objective function and noise terms, we propose a new analysis process to describe the properties of the local minimum. This analysis process is divided into two steps. In the first step, we focus on the structure of $L$ by ignoring all noise terms. Then we obtain the approximate region of the local minimum. In the second step, we use the exact version of Inverse Function Theorem as a bridge to connect the simplified objective function $\widetilde{L}$ and $L$. Finally, we prove the existence of the local minimum $(w_1^*, w_2^*)$ and describe the properties of this local minimum. 

Compared with linear SSL models, nonlinear alternatives are closer to the state-of-the-art SSL methods. The conclusions in this paper can guide us further in understanding the learning results of SSL methods and provide a theoretical basis for subsequent improvements.

\paragraph{Future work.} This paper analyzes a 1-layer nonlinear SSL model. After that, we plan to expand the scope of the analysis to a multi-layer nonlinear network. The multi-layer network analysis requires a more refined exploration of local minima. The weight matrix of each layer needs to be uniformly processed to analyze the landscape of the objective function, which we will do in the follow-up work.



\bibliography{ref}

\begin{thebibliography}{35}
\providecommand{\natexlab}[1]{#1}

\bibitem[{Allen-Zhu, Li, and Song(2019)}]{allen2019convergence}
Allen-Zhu, Z.; Li, Y.; and Song, Z. 2019.
\newblock A convergence theory for deep learning via over-parameterization.
\newblock In \emph{International Conference on Machine Learning}, 242--252. PMLR.

\bibitem[{Arora et~al.(2019)Arora, Khandeparkar, Khodak, Plevrakis, and Saunshi}]{arora2019theoretical}
Arora, S.; Khandeparkar, H.; Khodak, M.; Plevrakis, O.; and Saunshi, N. 2019.
\newblock A theoretical analysis of contrastive unsupervised representation learning.
\newblock \emph{arXiv preprint arXiv:1902.09229}.

\bibitem[{Bromley et~al.(1993)Bromley, Guyon, LeCun, S{\"a}ckinger, and Shah}]{bromley1993signature}
Bromley, J.; Guyon, I.; LeCun, Y.; S{\"a}ckinger, E.; and Shah, R. 1993.
\newblock Signature verification using a" siamese" time delay neural network.
\newblock \emph{Advances in neural information processing systems}, 6.

\bibitem[{Brutzkus and Globerson(2017)}]{brutzkus2017globally}
Brutzkus, A.; and Globerson, A. 2017.
\newblock Globally optimal gradient descent for a convnet with gaussian inputs.
\newblock In \emph{International conference on machine learning}, 605--614. PMLR.

\bibitem[{Bubeck et~al.(2015)}]{bubeck2015convex}
Bubeck, S.; et~al. 2015.
\newblock Convex optimization: Algorithms and complexity.
\newblock \emph{Foundations and Trends{\textregistered} in Machine Learning}, 8(3-4): 231--357.

\bibitem[{Caron et~al.(2020)Caron, Misra, Mairal, Goyal, Bojanowski, and Joulin}]{caron2020unsupervised}
Caron, M.; Misra, I.; Mairal, J.; Goyal, P.; Bojanowski, P.; and Joulin, A. 2020.
\newblock Unsupervised learning of visual features by contrasting cluster assignments.
\newblock \emph{Advances in Neural Information Processing Systems}, 33: 9912--9924.

\bibitem[{Chen et~al.(2020)Chen, Kornblith, Norouzi, and Hinton}]{chen2020simple}
Chen, T.; Kornblith, S.; Norouzi, M.; and Hinton, G. 2020.
\newblock A simple framework for contrastive learning of visual representations.
\newblock In \emph{International conference on machine learning}, 1597--1607. PMLR.

\bibitem[{Chen and He(2021)}]{DBLP:conf/cvpr/ChenH21}
Chen, X.; and He, K. 2021.
\newblock Exploring Simple Siamese Representation Learning.
\newblock In \emph{{IEEE} Conference on Computer Vision and Pattern Recognition, {CVPR} 2021, virtual, June 19-25, 2021}, 15750--15758. Computer Vision Foundation / {IEEE}.

\bibitem[{Devlin et~al.(2018)Devlin, Chang, Lee, and Toutanova}]{devlin2018bert}
Devlin, J.; Chang, M.-W.; Lee, K.; and Toutanova, K. 2018.
\newblock Bert: Pre-training of deep bidirectional transformers for language understanding.
\newblock \emph{arXiv preprint arXiv:1810.04805}.

\bibitem[{Du et~al.(2019)Du, Lee, Li, Wang, and Zhai}]{du2019gradient}
Du, S.; Lee, J.; Li, H.; Wang, L.; and Zhai, X. 2019.
\newblock Gradient descent finds global minima of deep neural networks.
\newblock In \emph{International conference on machine learning}, 1675--1685. PMLR.

\bibitem[{Du et~al.(2017)Du, Jin, Lee, Jordan, Singh, and Poczos}]{du2017gradient}
Du, S.~S.; Jin, C.; Lee, J.~D.; Jordan, M.~I.; Singh, A.; and Poczos, B. 2017.
\newblock Gradient descent can take exponential time to escape saddle points.
\newblock \emph{Advances in neural information processing systems}, 30.

\bibitem[{Grill et~al.(2020)Grill, Strub, Altch{\'e}, Tallec, Richemond, Buchatskaya, Doersch, Avila~Pires, Guo, Gheshlaghi~Azar et~al.}]{grill2020bootstrap}
Grill, J.-B.; Strub, F.; Altch{\'e}, F.; Tallec, C.; Richemond, P.; Buchatskaya, E.; Doersch, C.; Avila~Pires, B.; Guo, Z.; Gheshlaghi~Azar, M.; et~al. 2020.
\newblock Bootstrap your own latent-a new approach to self-supervised learning.
\newblock \emph{Advances in neural information processing systems}, 33: 21271--21284.

\bibitem[{HaoChen et~al.(2021)HaoChen, Wei, Gaidon, and Ma}]{haochen2021provable}
HaoChen, J.~Z.; Wei, C.; Gaidon, A.; and Ma, T. 2021.
\newblock Provable guarantees for self-supervised deep learning with spectral contrastive loss.
\newblock \emph{Advances in Neural Information Processing Systems}, 34: 5000--5011.

\bibitem[{HaoChen et~al.(2022)HaoChen, Wei, Kumar, and Ma}]{haochen2022beyond}
HaoChen, J.~Z.; Wei, C.; Kumar, A.; and Ma, T. 2022.
\newblock Beyond separability: Analyzing the linear transferability of contrastive representations to related subpopulations.
\newblock \emph{arXiv preprint arXiv:2204.02683}.

\bibitem[{He et~al.(2020)He, Fan, Wu, Xie, and Girshick}]{he2020momentum}
He, K.; Fan, H.; Wu, Y.; Xie, S.; and Girshick, R. 2020.
\newblock Momentum contrast for unsupervised visual representation learning.
\newblock In \emph{Proceedings of the IEEE/CVF conference on computer vision and pattern recognition}, 9729--9738.

\bibitem[{Jing et~al.(2021)Jing, Vincent, LeCun, and Tian}]{jing2021understanding}
Jing, L.; Vincent, P.; LeCun, Y.; and Tian, Y. 2021.
\newblock Understanding dimensional collapse in contrastive self-supervised learning.
\newblock \emph{arXiv preprint arXiv:2110.09348}.

\bibitem[{Kahan(1975)}]{kahan1975spectra}
Kahan, W. 1975.
\newblock Spectra of nearly Hermitian matrices.
\newblock \emph{Proceedings of the American Mathematical Society}, 48(1): 11--17.

\bibitem[{Khosla et~al.(2020)Khosla, Teterwak, Wang, Sarna, Tian, Isola, Maschinot, Liu, and Krishnan}]{khosla2020supervised}
Khosla, P.; Teterwak, P.; Wang, C.; Sarna, A.; Tian, Y.; Isola, P.; Maschinot, A.; Liu, C.; and Krishnan, D. 2020.
\newblock Supervised contrastive learning.
\newblock \emph{Advances in Neural Information Processing Systems}, 33: 18661--18673.

\bibitem[{Lee et~al.(2021)Lee, Lei, Saunshi, and Zhuo}]{lee2021predicting}
Lee, J.~D.; Lei, Q.; Saunshi, N.; and Zhuo, J. 2021.
\newblock Predicting what you already know helps: Provable self-supervised learning.
\newblock \emph{Advances in Neural Information Processing Systems}, 34: 309--323.

\bibitem[{Li and Liang(2018)}]{li2018learning}
Li, Y.; and Liang, Y. 2018.
\newblock Learning overparameterized neural networks via stochastic gradient descent on structured data.
\newblock \emph{Advances in neural information processing systems}, 31.

\bibitem[{Li and Yuan(2017)}]{li2017convergence}
Li, Y.; and Yuan, Y. 2017.
\newblock Convergence analysis of two-layer neural networks with relu activation.
\newblock \emph{Advances in neural information processing systems}, 30.

\bibitem[{Liu et~al.(2021)Liu, HaoChen, Gaidon, and Ma}]{liu2021self}
Liu, H.; HaoChen, J.~Z.; Gaidon, A.; and Ma, T. 2021.
\newblock Self-supervised learning is more robust to dataset imbalance.
\newblock \emph{arXiv preprint arXiv:2110.05025}.

\bibitem[{Radford et~al.(2019)Radford, Wu, Child, Luan, Amodei, Sutskever et~al.}]{radford2019language}
Radford, A.; Wu, J.; Child, R.; Luan, D.; Amodei, D.; Sutskever, I.; et~al. 2019.
\newblock Language models are unsupervised multitask learners.
\newblock \emph{OpenAI blog}, 1(8): 9.

\bibitem[{Rudin et~al.(1976)}]{rudin1976principles}
Rudin, W.; et~al. 1976.
\newblock \emph{Principles of mathematical analysis}, volume~3.
\newblock McGraw-hill New York.

\bibitem[{Tian(2017)}]{tian2017analytical}
Tian, Y. 2017.
\newblock An analytical formula of population gradient for two-layered relu network and its applications in convergence and critical point analysis.
\newblock In \emph{International conference on machine learning}, 3404--3413. PMLR.

\bibitem[{Tian(2022{\natexlab{a}})}]{tian2022deep}
Tian, Y. 2022{\natexlab{a}}.
\newblock Deep contrastive learning is provably (almost) principal component analysis.
\newblock \emph{arXiv preprint arXiv:2201.12680}.

\bibitem[{Tian(2022{\natexlab{b}})}]{tian2022understanding}
Tian, Y. 2022{\natexlab{b}}.
\newblock Understanding the Role of Nonlinearity in Training Dynamics of Contrastive Learning.
\newblock \emph{arXiv preprint arXiv:2206.01342}.

\bibitem[{Tian, Chen, and Ganguli(2021)}]{tian2021understanding}
Tian, Y.; Chen, X.; and Ganguli, S. 2021.
\newblock Understanding self-supervised learning dynamics without contrastive pairs.
\newblock In \emph{International Conference on Machine Learning}, 10268--10278. PMLR.

\bibitem[{Tian et~al.(2020)Tian, Yu, Chen, and Ganguli}]{tian2020understanding}
Tian, Y.; Yu, L.; Chen, X.; and Ganguli, S. 2020.
\newblock Understanding self-supervised learning with dual deep networks.
\newblock \emph{arXiv preprint arXiv:2010.00578}.

\bibitem[{Tosh, Krishnamurthy, and Hsu(2021)}]{tosh2021contrastive}
Tosh, C.; Krishnamurthy, A.; and Hsu, D. 2021.
\newblock Contrastive estimation reveals topic posterior information to linear models.
\newblock \emph{J. Mach. Learn. Res.}, 22: 281--1.

\bibitem[{Wang et~al.(2022)Wang, Wang, Shen, Fei, Li, Jin, Wu, Zhao, and Le}]{WangWSFLJWZL22}
Wang, Y.; Wang, H.; Shen, Y.; Fei, J.; Li, W.; Jin, G.; Wu, L.; Zhao, R.; and Le, X. 2022.
\newblock Semi-Supervised Semantic Segmentation Using Unreliable Pseudo-Labels.
\newblock In \emph{{IEEE/CVF} Conference on Computer Vision and Pattern Recognition, {CVPR} 2022, New Orleans, LA, USA, June 18-24, 2022}, 4238--4247. {IEEE}.

\bibitem[{Wen and Li(2021)}]{wen2021toward}
Wen, Z.; and Li, Y. 2021.
\newblock Toward understanding the feature learning process of self-supervised contrastive learning.
\newblock In \emph{International Conference on Machine Learning}, 11112--11122. PMLR.

\bibitem[{Wu et~al.(2020)Wu, Wang, Pino, and Gu}]{wu2020self}
Wu, A.; Wang, C.; Pino, J.; and Gu, J. 2020.
\newblock Self-supervised representations improve end-to-end speech translation.
\newblock \emph{arXiv preprint arXiv:2006.12124}.

\bibitem[{Zhang et~al.(2019)Zhang, Yu, Wang, and Gu}]{zhang2019learning}
Zhang, X.; Yu, Y.; Wang, L.; and Gu, Q. 2019.
\newblock Learning one-hidden-layer relu networks via gradient descent.
\newblock In \emph{The 22nd international conference on artificial intelligence and statistics}, 1524--1534. PMLR.

\bibitem[{Zhong et~al.(2022)Zhong, Tang, Chen, Peng, and Wang}]{zhong2022self}
Zhong, Y.; Tang, H.; Chen, J.; Peng, J.; and Wang, Y.-X. 2022.
\newblock Is Self-Supervised Learning More Robust Than Supervised Learning?
\newblock \emph{arXiv preprint arXiv:2206.05259}.

\end{thebibliography}

\appendix
\onecolumn

\section*{Appendix}

\section{Proof of Theorem \ref{thm:Convergence_noiseepexp_pro}}\label{sec:main_theorem_proof}

In this section, we first present the gradient of the nonlinear SSL model. Then we separate the core part of the gradient that affects the local minimum and the insignificant noise terms.
Recall that the objective function of the SSL model is
\begin{align}
    \min L=-\frac{1}{n}\sum_{i=1}^n\mathbb{E}_{\xi_{\text{aug}},\xi^{\prime}_{\text{aug}}}\left[\left\langle \sigma(W(x_i+\xi_{\text{aug}}),\sigma(W(x_i+\xi^{\prime}_{\text{aug}}))\right\rangle\right]+\alpha\left\|W\right\|_F^2\,,\notag
\end{align}
where $\alpha$ is the coefficient of regularizer, $W=\left[w_1,w_2\right]^{\top}\in \mathbb{R}^{2\times d}$ and $\xi_{\text{aug}},\xi'_{\text{aug}} \sim \mathcal{N}\left(0,\rho^2 I\right)$. 
For simplicity, we define $z^{(k)}$ as the $k$-th element of $z\in \mathbb{R}^d$.  We denote by $\frac{\partial L}{\partial w_{j}^{(k)}}$ the gradient of $w_{j}^{(k)}$, $\frac{\partial L}{\partial w_{j}}$ the gradient of $w_{j}, H(w_j)=\nabla^2_jL(w_1,w_2)=\frac{\partial^2 L}{\partial w_j^2}\in \mathbb{R}^{d\times d}$ the Hessian matrix, and $\nabla^3_jL(w_1,w_2)=\frac{\partial^3 L}{\partial w_j^3}\in \mathbb{R}^{d\times d \times d}, \forall j\in[2], k\in [d]$. 

We remark that each datapoint $x_i$ contains noise $\rho \xi_i$. Therefore, there are two kinds of noise terms: datapoint noise $\rho \xi_i$ and data augmentation noise $\xi_{\text{aug}},\xi'_{\text{aug}}$ in $L$. For the convenience of analysis, we first ignore $\rho\xi_i, \xi_{\text{aug}}, \xi^{\prime}_{\text{aug}}$ noise terms and do expectation over data distribution to get an simplified, analyzable version of the objective function:
\begin{align}
\min \widetilde{L}=-\mathbb{E}_{\widetilde{x}}[\langle \sigma(W\widetilde{x}),\sigma(W\widetilde{x})\rangle]+\alpha\|W\|_F^2\,, \notag
\end{align}
where $\widetilde{x}_i$ is the datapoint without without noise term $\rho \xi_i$. 

The detailed formulas of $\frac{\partial L}{\partial w_j^{(k)}}$ and $\frac{\partial \widetilde{L}}{\partial w_j^{(k)}},\forall j\in [2], k\in[d]$ will be presented in Appendix \ref{sec:detailed_formulas}. It is clear that the gradient of $w_1$ and $w_2$ are similar, so in part of the following lemmas, we mainly discuss the properties of $w_1$. Then using symmetry, we get the properties of $w_2$.

After simplifying the objective function, we analyze the existence of the local minimum $(\widetilde{w}_{1}^{*},\widetilde{w}_{2}^{*})$ of $\widetilde{L}$. We respectively  define 
\begin{align*}
    &\widetilde{D}_1(\tau) = \{\vec{x}\in \mathbb{R}^d|x^{(1)}\in[3.1,3.9], \tau x^{(2)}\in[9,+\infty), x^{(k)}=0, \forall k\in [3,d]\},\\
    &\widetilde{D}_2(\tau) = \{\vec{x}\in \mathbb{R}^d|x^{(1)}\in[-3.1,-3.9], \tau x^{(2)}\in[9,+\infty), x^{(k)}=0, \forall k\in [3,d]\}\,,
\end{align*}
as the region of $\widetilde{w}_1^{*}$ and $\widetilde{w}_2^{*}$.
Lemma \ref{lem:existence_no_noise} proves the existence of the point $(\widetilde{w}_1^{*}, \widetilde{w}_2^{*})$ which satisfies  $\frac{\partial \widetilde{L}}{\partial W}=0$ in the region $\widetilde{D}_1(\tau)\times \widetilde{D}_2(\tau)$. 

\begin{restatable}{lemma}{existencenonoise}\label{lem:existence_no_noise}
For $\alpha =1/800, \tau \ge 7$, the equation $\frac{\partial \widetilde{L}}{\partial W} = 0$ has a solution $(\widetilde{w}_{1}^{*},\widetilde{w}_{2}^{*})$, which satisfies $(\widetilde{w}_{1}^{*},\widetilde{w}_{2}^{*})\in \widetilde{D}_1(\tau)\times \widetilde{D}_2(\tau)$.
\end{restatable}

Lemma \ref{lem:strongconvex} uses Hessian matrix to characterize the landscape of $\widetilde{L}$ in the region $D_1^{B_0}(\tau)\times D_2^{B_0}(\tau)$ and show that $\widetilde{L}$ is locally strongly convex and $L_m$-smooth. The region $D_1^{B_0}(\tau)\times D_2^{B_0}(\tau)$ can be viewed as an ball center at  the initialization region $D_1(\tau)\times D_2(\tau)$ of $L$ mentioned in Theorem \ref{thm:Convergence_noiseepexp_pro}. We define $D_1^{B_0}(\tau)\times D_2^{B_0}(\tau)$ because the requirement of Def. \ref{def:locally_strong_convex_and_L_smooth} needs to be satisfied.
\begin{restatable}{lemma}{strongconvex}\label{lem:strongconvex}
For $\alpha =1/800, \tau\ge 7$, there is a region $D_1^{B_0}(\tau)\times D_2^{B_0}(\tau)$ s.t. $\widetilde{L}$ is $2\alpha$-strongly convex and $(2\alpha+\tau^2+1.5)$-smooth.
\end{restatable}

Using the results of Lemma \ref{lem:strongconvex}, we show that $(\widetilde{w}_1^{*},\widetilde{w}_2^{*})$ is a local minimum of $\widetilde{L}$. Furthermore, because of the definition of $\widetilde{D}_1(\tau)\times \widetilde{D}_2(\tau)$, it is clear that $(\widetilde{w}_1^{*},\widetilde{w}_2^{*})$ capture all important data distribution features. 

After completing the analysis of $\widetilde{L}$, we deal with the error terms because of the simplified process.
\paragraph{Error terms due to expectation.}We deal with the error terms due to the expectation operation over data distribution in $\widetilde{L}$. The objective function $\widehat{L}$ without expectation operation over data distribution is analyzed to deal with these error terms. $\widehat{L}$ can be written as:
\begin{align}
    \widehat{L}=-\frac{1}{n}\sum_{i=1}^n\langle \sigma(W\widetilde{x}_i),\sigma(W\widetilde{x}_i)\rangle+\alpha\|W\|_F^2\,,\notag
\end{align}
where $n=n_1+n_2+n_3+n_4.$ The detailed formulas of  $\frac{\partial \widehat{L}}{\partial w_j^{k}},\forall j\in [2], k\in[d]$ will be presented in Appendix \ref{sec:detailed_formulas}.

Lemma \ref{lem:dislocation_noise} focuses on the transformation process from $\widetilde{L}$ to $\widehat{L}$. This lemma proves the upper bound of gradient noise terms $\left\|\frac{\partial \widehat{L}-\widetilde{L}}{\partial w_1}\Big|_{w_1=\widetilde{w}_1^{*}}\right\|_2$, Hessian matrix noise terms $\left\|\frac{\partial^2\left(\widehat{L}-\widetilde{L}\right)}{\partial w_1^2}\right\|_F$, and $\|\nabla^3_1 \widehat{L}(w_1,w_2)\|_2$. 

\begin{restatable}{lemma}{dislocationnoise}\label{lem:dislocation_noise}
For $\tau = d^{\frac{1}{10}}$ and $n=d^2$, with probability $1-O\left(2^{-\frac{d^2}{10}}\right)$, $\left\|\frac{\partial \widehat{L}-\widetilde{L}}{\partial w_1}\Big|_{w_1=\widetilde{w}_1^{*}}\right\|_2\leq O(\tau n^{-\frac{9}{20}}), \left\|\frac{\partial^2\left(\widehat{L}-\widetilde{L}\right)}{\partial w_1^2}\right\|_F \leq O\left(\tau^2 n^{-\frac{9}{20}}\right)$ and $\|\nabla^3_1 \widehat{L}(w_1,w_2)\|_2\leq \tau^3$.
\end{restatable}

\paragraph{Datapoint and data augmentation noise terms.} We deal with datapoint noise terms $\rho\xi_i$ and data augmentation noise terms $\xi_{\text{aug}},\xi^{\prime}_{\text{aug}}$ based on the result of $\widehat{L}$. Therefore, we mainly focus on the noise terms of $L-\widehat{L}$.
Lemma \ref{lem:data_noise} proves the upper bound of gradient noise terms $\left\|\frac{\partial L-\widehat{L}}{\partial w_1}\Big|_{w_1=\widetilde{w}_1^{*}}\right\|_2$, Hessian matrix noise terms $\left\|\frac{\partial^2\left(L-\widehat{L}\right)}{\partial w_1^2}\right\|_F$, and $\|\nabla^3_1 L(w_1, w_2)\|_2$. 

\begin{restatable}{lemma}{datanoise}\label{lem:data_noise}
When $w_1\in D_1^{B_0}(\tau), \tau = d^{\frac{1}{10}}, \rho= 1/d^{1.5}$ and $n=d^2$, with probability $1-O\left(e^{-d^{\frac{1}{10}}}\right)$ and large enough $d$, $\left\|\frac{\partial L-\widehat{L}}{\partial w_1}\Big|_{w_1=\widetilde{w}_1^{*}}\right\|_2\leq O(\rho^{\frac{13}{15}}d^{\frac{6}{10}}), \left\|\frac{\partial^2\left(L-\widehat{L}\right)}{\partial w_1^2}\right\|_F\leq O(\rho^{\frac{4}{5}} d^\frac{11}{10})$ and $\|\nabla^3_1 L(w_1, w_2)\|_2\leq \Theta(\sqrt{d})$.
\end{restatable}

Lemma \ref{lem:dislocation_noise} and Lemma \ref{lem:data_noise} deal with error terms due to expectation, datapoint noise terms, and data augmentation noise terms, allowing us to convert $\widetilde{L}$ to $L$. Finally, Lemma \ref{lem:strongconvex_noise} combines the above result to characterize the landscape of $L$ in the region $D_1^{B_0}(\tau)\times D_2^{B_0}(\tau)$.

\begin{restatable}{lemma}{strongconvexnoise}\label{lem:strongconvex_noise}
For $\tau= d^{\frac{1}{10}}, \rho= \frac{1}{d^{1.5}}$ and $n=d^2$, when $(w_1,w_2)\in D_1^{B_0}(\tau)\times D_2^{B_0}(\tau)$, with probability $1-O\left(e^{-d^{\frac{1}{10}}}\right)$ and large enough $d$, $L$ is $(2\alpha-\rho^{\frac{4}{5}}d^{\frac{11}{10}})$-strongly convex and $(2\alpha+\tau^2+1.5+\rho^{\frac{4}{5}}d^{\frac{11}{10}})$-smooth. At the same time, $\nabla^2 L(w_1)$ is $L_H$-Lipschitz continuous Hessian where $L_H=\Theta(\sqrt{d})$.
\end{restatable}

Using Lemma \ref{lem:strongconvex_noise}, we can prove that the local minimum $(w_1^*, w_2^*)$ is not far away from the solution $(\widetilde{w}_1^*, \widetilde{w}_2^*)$ in Lemma \ref{lem:existence_no_noise}. Hence the projection $|\Pi e_j|$ of $e_j, \forall j\in[2]$ on the plant spanned by $w_1^*$ and $w_2^*$ will be very close to $1$. Lemma \ref{lem:projection} formally describes this phenomenon.
\begin{lemma}\label{lem:projection}
 Let $(\widetilde{w}_1^*, \widetilde{w}_2^*)$ be the solution in Lemma \ref{lem:existence_no_noise} and $W=[w_1, w_2]\in \mathbb{R}^{2\times d}$. When $\|w_1-\widetilde{w}_1^*\|_2\leq d^{-\frac{1}{2}}$ and $\|w_2-\widetilde{w}_2^*\|_2\leq d^{-\frac{1}{2}}, |\Pi_W e_j|\geq 1-O(\tau^3 d^{-\frac{1}{2}}), \forall j\in [2].$
\end{lemma}

The proof details of the above lemmas are presented in Appendix \ref{sec:useful_lemma}.

In the proof process of Theorem \ref{thm:Convergence_noiseepexp_pro}, we adopt the proof idea of the Inverse Function Theorem \cite{rudin1976principles} to prove the existence of the local minimum $(w_1^*, w_2^*)$ of $L$. Our intuition is that since $(\widetilde{w}_1^*, \widetilde{w}_2^*)$ has good properties (Lemma \ref{lem:strongconvex}), $\frac{\partial L}{\partial w_1}$ should be one-to-one in the neighborhood of $(\widetilde{w}_1^*, \widetilde{w}_2^*)$. Combined with the fact that $\left\|\frac{\partial L}{\partial w_1}\big|_{w_1=\widetilde{w}_1^*}\right\|_2$ is small, the existence of $(w_1^*, w_2^*)$ can be proved. Lemma \ref{lem:inverse} formally describes this phenomenon. Define $B(x,r)$ as the open ball with radius $r$ centered at $x$.
\begin{restatable}{lemma}{inverse}\label{lem:inverse}
Suppose $f$ is a differentiable function mapping $\mathbb{R}^d$ into $\mathbb{R}^d$, $f'(a)$ is invertible for some $a\in \mathbb{R}^d$. If $f'$ is L-Lipschitz, there is an open set $B\left(a, \frac{1}{2\|A^{-1}\|_2L}\right):=B_1$ in $\mathbb{R}^d$, $f$ is one-to-one on $B_1$. Moreover, $B\left(f(a),\frac{1}{4L\|A^{-1}\|_2^2}\right)\subset f(B_1)$.
\end{restatable}
\begin{proof}[Proof]
Let $f'(a)=A$. We choose $\lambda$ so that
\begin{align}
    2\lambda\|A^{-1}\|_2=1\,. \label{lambda_cal}
\end{align}

Let $r=\frac{\lambda}{L}$. Since $f'$ is $L$-Lipschitz, by choosing $B_1=B(a,r)$, we have:
\begin{align}
    \|f'(x)-A\|_2<\lambda , \forall x\in B_1\,. \label{radius_cal}
\end{align}

Define
\begin{align}
    \phi(x) =x+A^{-1}(y-f(x)), \forall x\in \mathbb{R}^d\,, \notag
\end{align}
as a function which associate to each $y\in \mathbb{R}^d$.

Note that $f(x)=y$ if and only if $x$ is a fixed point of $\phi$. Since $\phi'(x)=I-A^{-1}f'(x)=A^{-1}(A-f'(x))$, Eq. (\ref{lambda_cal}) and Eq. (\ref{radius_cal}) imply that
\begin{align}
    \|\phi'(x)\|_2<\frac{1}{2}, \forall x\in B_1\,.\notag
\end{align}

Hence 
\begin{align}
    \|\phi(x_1)-\phi(x_2)\|_2\leq \frac{1}{2}\|x_1-x_2\|_2,  \forall x_1, x_2\in B_1\,,\label{contraction}
\end{align}

according to Theorem 9.19 in~\citet{rudin1976principles}. It follows that $\phi$ has at most one fixed point in $B_1$, so that $f(x)=y$ for at most one $x\in B_1$. Thus $f$ is one-to-one in $B_1$.

Next, let $V=f(B_1)$. We will show that $y\in V$ whenever 
\begin{align}
    \|y-b\|_2<\lambda r\,.\notag
\end{align}

Define $\bar{B_1}$ as the closure of $B_1$. For any $x\in\bar{B_1}$, we have:
\begin{align}
    \|\varphi(x)-a\|_2&\leq \|\varphi(x)-\varphi(a)\|_2+\|\varphi(a)-a\|_2\notag\\
    &\leq \frac{1}{2}\|x-a\|_2+\|A^{-1}(y-f(a))\|_2\notag\\
    &\leq \frac{1}{2}r+\lambda\|A^{-1}\|_2r=r\,.\notag
\end{align}

The above inequality shows that $\varphi(x)\in B_1$. Eq. (\ref{contraction}) shows $\varphi$ is a contraction of $\bar{B}_1$ into $\bar{B}_1$. By the contraction principle, we know $\varphi$ has a unique fixed point, which completes the proof.
\end{proof}

\begin{proof}[Proof of Theorem 1]
To prove the existence of $\frac{\partial L}{\partial w_{1}}=0$, we use the modified version of the Inverse Function Theorem (Lemma \ref{lem:inverse}).
Lemma \ref{lem:existence_no_noise} shows that the local minimum of $\widetilde{L}$ is
\begin{center}
    $\widetilde{w}_{1}^{*(1)}\in[3.1,3.9]$, $\widetilde{w}_{1}^{*(1)}=-\widetilde{w}_{2}^{*(1)}$, $\tau \widetilde{w}_{1}^{*(2)} = \tau \widetilde{w}_{2}^{*(2)}\ge 9$.
\end{center}

From Lemma \ref{lem:dislocation_noise} and Lemma \ref{lem:data_noise} we know
\begin{align}
\left\|\frac{\partial L}{\partial w_1}\bigg|_{w_1=\widetilde{w}_1^*}\right\|_2\leq \left\|\frac{\partial \widehat{L}-\widetilde{L}}{\partial w_1}\Big|_{w_1=\widetilde{w}_1^{*}}\right\|_2+\left\|\frac{\partial L-\widehat{L}}{\partial w_1}\Big|_{w_1=\widetilde{w}_1^{*}}\right\|_2=O(\tau n^{-\frac{9}{20}})+O(\rho^{\frac{13}{15}}d^{\frac{6}{10}})=O(d^{-\frac{7}{10}})\,. \notag
\end{align}
 
In order to use Lemma \ref{lem:inverse}, we take $f=\frac{\partial L}{\partial w_1}$, $a=\widetilde{w}_1^{*}$, $L=L_H$, $A=H(\widetilde{w}_1^{*})$, $B_1= B(a,\frac{1}{2\|A^{-1}\|_2L})$, and the ball $B(f(a),r_0)\subset f(B_1)$, where 
$r_0=\frac{1}{4L\|A^{-1}\|_2^2}$.
From Lemma \ref{lem:strongconvex_noise}, with probability $1-O(e^{-d^{\frac{1}{10}}})$, we have $\frac{1}{\|A\|_2}=\Theta(1)$ because of $\alpha$ is a constant and $L_H=\Theta(\sqrt{d})$. Therefore 
 \begin{align}
 r_0=\frac{1}{4L\|A^{-1}\|_2^2}=\Theta(d^{-\frac{1}{2}})\geq \left\|\frac{\partial L}{\partial w_1}\bigg|_{w_1=\widetilde{w}_1^{*}}\right\|_2=O(d^{-\frac{7}{10}})\notag\,.
\end{align}
 
The latest inequality shows that the equation $\frac{\partial L}{\partial w_1}=0$  has a solution $w_1^*$ which is close to the local minimum in Lemma \ref{lem:existence_no_noise}:
\begin{align}
    \|w_1^*-\widetilde{w}_1^{*}\|_2\leq O(d^{-\frac{1}{2}})\,.\notag
\end{align}

We can use the same argument to get $w_2^*$ which solves $\frac{\partial L}{\partial w_2}=0$.
It is also easy to check the locally $\mu$-strongly convexity and $L_m$-smoothness of $L$ in the initialization region by Lemma \ref{lem:strongconvex_noise}, so $(w_1^*,w_2^*)$ is indeed a local minimum of $L$. 

In order to get the convergence rate, we need to show that the initialization region satisfies the requirement of Lemma \ref{lem:convergence}. We know that $w_{1}^{*(1)}\in[3.1,3.9], \tau w_{1}^{*(2)}\ge 9$ and $w_{1}^{*(k)}\leq O(d^{-\frac{1}{2}}), \forall k\in[3,d]$. It is clear that the initialization region $D_1(\tau)\times D_2(\tau)$ and the region $D_1^{B_0}(\tau)\times D_2^{B_0}(\tau)$ related to $B_0$ meet the requirement of Def. \ref{def:locally_strong_convex_and_L_smooth}. Hence, we can use Lemma \ref{lem:convergence} to get the convergence rate in Theorem \ref{thm:Convergence_noiseepexp_pro}.

To finish the proof, we need to calculate the length of  projection of $e_1$ and $e_2$ on the plane spanned by $w_1^*$ and $w_2^*$. Using Lemma \ref{lem:projection}, it is clear that $|\Pi e_j| \geq 1-O(\tau^3 d^{-\frac{1}{2}}), \forall j\in[2]$.

The above process completes the proof of Theorem \ref{thm:Convergence_noiseepexp_pro}. 

\end{proof}

\begin{fact}[Modified process for tanh activation function]
We define 
\begin{align*}
    \widetilde{D}_1^{\sigma_2}(\tau) &= \{\vec{x}\in \mathbb{R}^d|x^{(1)}\in[2.75,3.1], \tau x^{(2)}\in[6.1,+\infty), x^{(k)}=0, \forall k\in[3,d]\}\,,\\
    \widetilde{D}_2^{\sigma_2}(\tau) &= \{\vec{x}\in \mathbb{R}^d|x^{(1)}\in[-3.1,-2.75], \tau x^{(2)}\in[6.1,+\infty), x^{(k)}=0, \forall k\in[3,d]\}
\end{align*}
as the region of $\widetilde{w}_1^{\sigma_2*}$ and $\widetilde{w}_2^{\sigma_2*}$. Further, we define 
\begin{align*}
    D_1^{\sigma_2}(\tau)&=\{\vec{x}\in \mathbb{R}^d|x^{(1)}\in(2.75,3.1), \tau x^{(2)}\in(5.75,6.1), x^{(k)} \in(-\frac{3}{d^{0.49}},\frac{3}{d^{0.49}}), \forall k\in [3,d]\}\,,\\
    D_2^{\sigma_2}(\tau)&=\{\vec{x}\in \mathbb{R}^d|x^{(1)}\in(-3.1,-2.75), \tau x^{(2)}\in(5.75,6.1), x^{(k)} \in(-\frac{3}{d^{0.49}},\frac{3}{d^{0.49}}), \forall k\in [3,d]\}
\end{align*}
as the initialization region of $w_1^{\sigma_2}$ and $w_2^{\sigma_2}$. Then it is easy to get similar results of Theorem \ref{thm:Convergence_noiseepexp_pro} by using the same process of Lemma \ref{lem:strongconvex_noise} to deal with noise terms.

\end{fact}

\section{Proof of Theorem \ref{thm:SLtheoremep}}\label{sec:SLproof} 
In this section, we present the proof detail of Theorem \ref{thm:SLtheoremep}.
~\\
\textbf{Theorem 2.}(restated)\textit{ Let $w_1^{\text{SL},*}$ and $w_2^{\text{SL},*}$ be the optimal solution of $L_{\text{SL}}$. Then with probability $1-O(e^{-d^\frac{1}{10}}),$
\begin{align}
    \left(w_{1}^{\text{SL},*(2)}\right)^2+\left(w_{2}^{\text{SL},*(2)}\right)^2\leq O\left(\rho d^{\frac{1}{10}}\right)\,.\notag
\end{align}
When we choose $\rho = 1/d^{1.5}, \left(w_{1}^{\text{SL},*(2)}\right)^2+\left(w_{2}^{\text{SL},*(2)}\right)^2\leq O\left(1/d^{1.4}\right)$.}
\begin{proof}[Proof]
Using the tail bound of standard Gaussian random and the union bound, $|\xi_i|\leq d^{\frac{1}{10}}, \forall i\in[n]$ holds with probability $1-O(e^{-d^{\frac{1}{10}}})$.
We choose a reference solution $w_1^r=2e_1$ and $w_2^r=-2e_1$. Using the Lagrange’s Mean Value Theorem, we have
\begin{align*}
    \sigma\left((w_1^r)^{\top}(e_1+\rho\xi)\right) &= \sigma(2+2\rho\xi^{(1)})\notag\\&=\sigma(2)+\sigma'(\theta)(2\rho\xi^{(1)})\notag\\&\ge \sigma(2)-2\rho d^{\frac{1}{10}}\,.
\end{align*}

Similarly, 
\begin{align*}
    \sigma\left((w_2^r)^{\top}(e_1+\rho\xi)\right) &= \sigma(-2-2\rho\xi^{(1)})\notag\\&\leq \sigma(-2)+2\rho d^{\frac{1}{10}}\,.
\end{align*}

Recall that the optimal solution of Eq. (\ref{objective_SL_ep}) is $w_1^{\text{SL},*}$ and $w_2^{\text{SL},*}$. Because of the definition of $L_{\text{SL}}$, we know that
\begin{align*}
    \|w_1^{\text{SL},*}\|_2^2+\|w_2^{\text{SL},*}\|_2^2\leq \|w_1^r\|_2^2+\|w_1^r\|_2^2 =8\,.
\end{align*}

Because of the margin constraint of $L_{\text{SL}}$, we need to guarantee $\sigma(w_1^{\text{SL},*}x)-\sigma(w_2^{\text{SL},*}x)\ge 2\sigma(2)-1-\Theta(\rho d^{\frac{1}{10}})$.
At least, $w_1^{\text{SL},*},w_2^{\text{SL},*}$ need to meet this requirement when $x=e_1+\rho\xi$. Hence the following inequality needs to be satisfied:
\begin{align}
    \sigma(w_{1}^{\text{SL},*(1)}+\rho (w_1^{\text{SL},*})^{\top}\xi)\ge \sigma(w_{2}^{\text{SL},*(1)}+\rho (w_2^{\text{SL},*})^{\top}\xi)+2\sigma(2)-1-\Theta(\rho d^{\frac{1}{10}}). \label{w1margin_ep}
\end{align}

With the Lagrange’s Mean Value Theorem, we have
\begin{align*}
    \sigma(w_{1}^{\text{SL},*(1)}+\rho (w_1^{\text{SL},*})^{\top}\xi)&\leq  \sigma(w_{1}^{\text{SL},*(1)})+2\rho d^{\frac{1}{10}}\,, \\
\sigma(w_{2}^{\text{SL},*(1)}+\rho (w_2^{\text{SL},*})^{\top}\xi)&\ge  \sigma(w_{2}^{\text{SL},*(1)})-2\rho d^{\frac{1}{10}}\,. 
\end{align*}

It is clear that
\begin{align*}
    \sigma(w_{1}^{\text{SL},*(1)}+\rho (w_1^{\text{SL},*})^{\top}\xi)- \sigma(w_{2}^{\text{SL},*(1)}+\rho (w_2^{\text{SL},*})^{\top}\xi)\leq \sigma(w_{1}^{\text{SL},*(1)})-\sigma(w_{2}^{\text{SL},*(1)})+4\rho d^{\frac{1}{10}}\,.
\end{align*}

Without loss of generality, we suppose $w_{2}^{\text{SL},*(1)}\leq0\leq w_{1}^{\text{SL},*(1)}$ and $|w_{2}^{\text{SL},*(1)}|\leq |w_{1}^{\text{SL},*(1)}|$. Then a necessary requirement of Eq. (\ref{w1margin_ep}) can be written as
\begin{align}
    \sigma(w_{1}^{\text{SL},*(1)})+\sigma(-w_{2}^{\text{SL},*(1)})\ge 2\sigma(2)-\Theta(\rho d^{\frac{1}{10}})\,. \notag
\end{align}

Because $\sigma'(x)$ is deceasing when $x\ge0$ and $\left(w_{1}^{\text{SL},*(1)}\right)^2+\left(w_{2}^{\text{SL},*(1)}\right)^2\leq 8,$
\begin{align*}
    w_{1}^{\text{SL},*(1)} = 2, 
    \sigma(-w_{2}^{\text{SL},*(1)})\ge \sigma(2)-\Theta(\rho d^{\frac{1}{10}})\,.
\end{align*}

Hence we need $-w_{2}^{\text{SL},*(1)}\geq 2-\Theta(\rho d^{\frac{1}{10}})$.
With $\|w_1^{\text{SL},*}\|_2^2+\|w_2^{\text{SL},*}\|_2^2 \leq 8$, it is clear that $\left(w_{1}^{\text{SL},*(2)}\right)^2+\left(w_{2}^{\text{SL},*(2)}\right)^2\leq O(\rho d^{\frac{1}{10}})$.

Above all, we have finished the proof of Theorem \ref{thm:SLtheoremep}.
\end{proof}

\section{Detailed Formulas}\label{sec:detailed_formulas}
In this section, we present the detailed formulas of $\frac{\partial L}{\partial w_j^{(k)}}, \frac{\partial \widetilde{L}}{\partial w_j^{(k)}}$, and $\frac{\partial \widehat{L}}{\partial w_j^{(k)}}, \forall j\in[2], k\in [d]$. We take the gradient of $w_{1}$ as an example. Replacing $w_{1}^{(k)}$ with $w_{2}^{(k)}, \forall k\in[d]$, we get $\frac{\partial \widetilde{L}}{\partial w_{2}^{(k)}}, \forall k\in [d]$.
~\\
\textbf{The detailed formulas of $\frac{\partial L}{\partial w_j^{(k)}}, \forall j\in[2], k\in [d]$.}  
\begin{align}
    &\frac{\partial L}{\partial w_{1}^{(1)}}=\notag\\ & -\frac{1}{n}\bigg( \sum_{i=1}^{n_1} \mathbb{E}_{\xi_{\text{aug}},\xi^{\prime}_{\text{aug}}}\left[(1+\rho\xi_{i}^{(1)}+\xi_{\text{aug}}^{(1)})\sigma^{\prime}\left(w_{1}^{(1)}+w_1^{\top}(\rho\xi_{i}+\xi_{\text{aug}})\right)\sigma\left(w_{1}^{(1)}+w_1^{\top}(\rho\xi_{i}+\xi_{\text{aug}}^{\prime})\right)\right.\notag\\&\left.\qquad\qquad\qquad\quad\quad\quad\quad+(1+\rho\xi_{i}^{(1)}+\xi_{\text{aug}}^{\prime(1)})\sigma\left(w_{1}^{(1)}+w_1^{\top}(\rho\xi_{i}+\xi_{\text{aug}})\right)\sigma^{\prime}\left(w_{1}^{(1)}+w_1^{\top}(\rho\xi_{i}+\xi_{\text{aug}}^{\prime})\right)\right]\notag\\
    &\quad+\sum_{i=n_1+1}^{n_1+n_2} \mathbb{E}_{\xi_{\text{aug}},\xi^{\prime}_{\text{aug}}}\left[(1+\rho\xi_{i}^{(1)}+\xi_{\text{aug}}^{(1)})\sigma^{\prime}\left(w_{1}^{(1)}+\tau w_{1}^{(2)}+w_1^{\top}(\rho\xi_{i}+\xi_{\text{aug}})\right)\sigma\left(w_{1}^{(1)}+\tau w_{1}^{(2)}+w_1^{\top}(\rho\xi_{i}+\xi_{\text{aug}}^{\prime})\right)\right.\notag\\&\left.\qquad\qquad\qquad\quad\quad\qquad+(1+\rho\xi_{i}^{(1)}+\xi_{\text{aug}}^{\prime(1)})\sigma\left(w_{1}^{(1)}+\tau w_{1}^{(2)}+w_1^{\top}(\rho\xi_{i}+\xi_{\text{aug}})\right)\sigma^{\prime}\left(w_{1}^{(1)}+\tau w_{1}^{(2)}+w_1^{\top}(\rho\xi_{i}+\xi_{\text{aug}}^{\prime})\right)\right]\notag\\
  &\quad+\sum_{i=n_1+n_2+1}^{n_1+n_2+n_3} \mathbb{E}_{\xi_{\text{aug}},\xi^{\prime}_{\text{aug}}}\left[(-1+\rho\xi_{i}^{(1)}+\xi_{\text{aug}}^{(1)})\sigma^{\prime}\left(-w_{1}^{(1)}+w_1^{\top}(\rho\xi_{i}+\xi_{\text{aug}})\right)\sigma\left(-w_{1}^{(1)}+w_1^{\top}(\rho\xi_{i}+\xi_{\text{aug}}^{\prime})\right)\right.\notag\\&\left.\qquad\qquad\qquad\qquad\qquad\quad+(-1+\rho\xi_{i}^{(1)}+\xi_{\text{aug}}^{\prime(1)})\sigma\left(-w_{1}^{(1)}+w_1^{\top}(\rho\xi_{i}+\xi_{\text{aug}})\right)\sigma^{\prime}\left(-w_{1}^{(1)}+w_1^{\top}(\rho\xi_{i}+\xi_{\text{aug}}^{\prime})\right)\right]\notag\\
  &\quad+\sum_{i=n-n_4+1}^{n} \mathbb{E}_{\xi_{\text{aug}},\xi^{\prime}_{\text{aug}}}\left[(-1+\rho\xi_{i}^{(1)}+\xi_{\text{aug}}^{(1)})\sigma^{\prime}\left(-w_{1}^{(1)}+\tau w_{1}^{(2)}+w_1^{\top}(\rho\xi_{i}+\xi_{\text{aug}})\right)\sigma\left(-w_{1}^{(1)}+\tau w_{1}^{(2)}+w_1^{\top}(\rho\xi_{i}+\xi_{\text{aug}}^{\prime})\right)\right.\notag\\&\left.\qquad\qquad\qquad\quad\quad+(-1+\rho\xi_{i}^{(1)}+\xi_{\text{aug}}^{\prime(1)})\sigma\left(-w_{1}^{(1)}+\tau w_{1}^{(2)}+w_1^{\top}(\rho\xi_{i}+\xi_{\text{aug}})\right)\sigma^{\prime}\left(-w_{1}^{(1)}+\tau w_{1}^{(2)}+w_1^{\top}(\rho\xi_{i}+\xi_{\text{aug}}^{\prime})\right)\right]\bigg)\notag\\&\quad+2\alpha w_{1}^{(1)}\,,\label{gradient}\\
      &\frac{\partial L}{\partial w_{1}^{(2)}}=\notag\\ &-\frac{1}{n}\bigg( \sum_{i=1}^{n_1} \mathbb{E}_{\xi_{\text{aug}},\xi^{\prime}_{\text{aug}}}\left[(\rho\xi_{i}^{(2)}+\xi_{\text{aug}}^{(2)})\sigma^{\prime}\left(w_{1}^{(1)}+w_1^{\top}(\rho\xi_{i}+\xi_{\text{aug}})\right)\sigma\left(w_{1}^{(1)}+w_1^{\top}(\rho\xi_{i}+\xi_{\text{aug}}^{\prime})\right)\right.\notag\\&\left.\qquad\qquad\qquad\quad\quad\quad\quad+(\rho\xi_{i}^{(2)}+\xi_{\text{aug}}^{\prime(2)})\sigma\left(w_{1}^{(1)}+w_1^{\top}(\rho\xi_{i}+\xi_{\text{aug}})\right)\sigma^{\prime}\left(w_{1}^{(1)}+w_1^{\top}(\rho\xi_{i}+\xi_{\text{aug}}^{\prime})\right)\right]\notag\\
    &\quad+\sum_{i=n_1+1}^{n_1+n_2} \mathbb{E}_{\xi_{\text{aug}},\xi^{\prime}_{\text{aug}}}\left[(\tau+\rho\xi_{i}^{(2)}+\xi_{\text{aug}}^{(2)})\sigma^{\prime}\left(w_{1}^{(1)}+\tau w_{1}^{(2)}+w_1^{\top}(\rho\xi_{i}+\xi_{\text{aug}})\right)\sigma\left(w_{1}^{(1)}+\tau w_{1}^{(2)}+w_1^{\top}(\rho\xi_{i}+\xi_{\text{aug}}^{\prime})\right)\right.\notag\\&\left.\qquad\qquad\qquad\quad\quad\qquad+(\tau+\rho\xi_{i}^{(2)}+\xi_{\text{aug}}^{\prime(2)})\sigma\left(w_{1}^{(1)}+\tau w_{1}^{(2)}+w_1^{\top}(\rho\xi_{i}+\xi_{\text{aug}})\right)\sigma^{\prime}\left(w_{1}^{(1)}+\tau w_{1}^{(2)}+w_1^{\top}(\rho\xi_{i}+\xi_{\text{aug}}^{\prime})\right)\right]\notag\\
  &\quad+\sum_{i=n_1+n_2+1}^{n_1+n_2+n_3} \mathbb{E}_{\xi_{\text{aug}},\xi^{\prime}_{\text{aug}}}\left[(\rho\xi_{i}^{(2)}+\xi_{\text{aug}}^{(2)})\sigma^{\prime}\left(-w_{1}^{(1)}+w_1^{\top}(\rho\xi_{i}+\xi_{\text{aug}})\right)\sigma\left(-w_{1}^{(1)}+w_1^{\top}(\rho\xi_{i}+\xi_{\text{aug}}^{\prime})\right)\right.\notag\\&\left.\qquad\qquad\qquad\qquad\qquad\quad+(\rho\xi_{i}^{(2)}+\xi_{\text{aug}}^{\prime(2)})\sigma\left(-w_{1}^{(1)}+w_1^{\top}(\rho\xi_{i}+\xi_{\text{aug}})\right)\sigma^{\prime}\left(-w_{1}^{(1)}+w_1^{\top}(\rho\xi_{i}+\xi_{\text{aug}}^{\prime})\right)\right]\notag\\
  &\quad+\sum_{i=n-n_4+1}^{n} \mathbb{E}_{\xi_{\text{aug}},\xi^{\prime}_{\text{aug}}}\left[(\tau+\rho\xi_{i}^{(2)}+\xi_{\text{aug}}^{(2)})\sigma^{\prime}\left(-w_{1}^{(1)}+\tau w_{1}^{(2)}+w_1^{\top}(\rho\xi_{i}+\xi_{\text{aug}})\right)\sigma\left(-w_{1}^{(1)}+\tau w_{1}^{(2)}+w_1^{\top}(\rho\xi_{i}+\xi_{\text{aug}}^{\prime})\right)\right.\notag\\&\left.\qquad\qquad\qquad\quad\quad+(\tau+\rho\xi_{i}^{(2)}+\xi_{\text{aug}}^{\prime(2)})\sigma\left(-w_{1}^{(1)}+\tau w_{1}^{(2)}+w_1^{\top}(\rho\xi_{i}+\xi_{\text{aug}})\right)\sigma^{\prime}\left(-w_{1}^{(1)}+\tau w_{1}^{(2)}+w_1^{\top}(\rho\xi_{i}+\xi_{\text{aug}}^{\prime})\right)\right]\bigg)\notag\\&\quad+2\alpha w_{1}^{(2)}\,,\notag
\end{align}
where $n =n_1+n_2+n_3+n_4$.

~\\
\textbf{The detailed formulas of $\frac{\partial \widetilde{L}}{\partial w_j^{(k)}}, \forall j\in[2], k\in [d] $.}

For simplicity, we define $h(x)= \sigma'(x)\sigma(x)=\sigma^2(x)(1-\sigma(x))$ as the gradient of $\frac{1}{2}\sigma^2(x)$. Then the gradient of $\widetilde{L}$ can be written as 
\begin{align}
    \frac{\partial \widetilde{L}}{\partial w_{1}^{(1)}} &= -\frac{1}{2}(h(w_{1}^{(1)}+\tau w_{1}^{(2)})-h(-w_{1}^{(1)}+\tau w_{1}^{(2)})+\sigma(w_{1}^{(1)})\sigma(-w_{1}^{(1)})(\sigma(w_{1}^{(1)})-\sigma(-w_{1}^{(1)})))+2\alpha w_{1}^{(1)},\notag\\
    \frac{\partial \widetilde{L}}{\partial w_{1}^{(2)}} &= -\frac{\tau}{2}(h(w_{1}^{(1)}+\tau w_{1}^{(2)})+h(-w_{1}^{(1)}+\tau w_{1}^{(2)})+2\alpha w_{1}^{(2)}\,,\notag\\
    \frac{\partial \widetilde{L}}{\partial w_{1}^{(k)}} &= 2\alpha w_{1}^{(k)}\,,  \forall k\ge 3\,.\notag
\end{align}

~\\
\textbf{The detailed formulas of $\frac{\partial \widehat{L}}{\partial w_j^{(k)}}, \forall j\in[2], k\in [d] $.}
\begin{align}
    \frac{\partial \widehat{L}}{\partial w_{1}^{(1)}} &= -\frac{1}{n}\left( \sum_{i=1}^{n_1} 2\sigma'(w_{1}^{(1)})\sigma(w_{1}^{(1)})
    +\sum_{i=1}^{n_2} 2\sigma'(w_{1}^{(1)}+\tau w_{1}^{(2)})\sigma(w_{1}^{(1)}+\tau w_{1}^{(2)})\right.\notag\\
    &\left.\qquad\qquad\quad-\sum_{i=1}^{n_3} 2\sigma'(-w_{1}^{(1)})\sigma(-w_{1}^{(1)})
    -\sum_{i=1}^{n_4} 2\sigma'(-w_{1}^{(1)}+\tau w_{1}^{(2)})\sigma(-w_{1}^{(1)}+\tau w_{1}^{(2)})\right)+2\alpha w_{1}^{(1)}\,, \notag\\
    \frac{\partial \widehat{L}}{\partial w_{1}^{(2)}} &= -\frac{\tau}{n}\left(\sum_{i=1}^{n_2} 2\sigma'(w_{1}^{(1)}+\tau w_{1}^{(2)})\sigma(w_{1}^{(1)}+\tau w_{1}^{(2)})+\sum_{i=1}^{n_4} 2\sigma'(-w_{1}^{(1)}+\tau w_{1}^{(2)})\sigma(-w_{1}^{(1)}+\tau w_{1}^{(2)})\right)+2\alpha w_{1}^{(2)}\,, \notag\\
    \frac{\partial \widehat{L}}{\partial w_{1}^{(k)}} &= 2\alpha w_{1}^{(k)},  \forall k\ge 3\,,\notag
\end{align}
where $n =n_1+n_2+n_3+n_4$.

\section{Auxiliary Lemmas}\label{sec:useful_lemma}
~\\
\textbf{Lemma 1.} (restated) \textit{For $\alpha =1/800, \tau \ge 7$, the equation $\frac{\partial \widetilde{L}}{\partial W} = 0$ has a solution $(\widetilde{w}_{1}^{*},\widetilde{w}_{2}^{*})$, which satisfies $(\widetilde{w}_{1}^{*},\widetilde{w}_{2}^{*})\in \widetilde{D}_1(\tau)\times \widetilde{D}_2(\tau)$.}
\begin{proof}[Proof] 
It is clear that $w_1\in \widetilde{D}_1(\tau)$ contains only two non-zero elements $w_{1}^{(1)}$ and $\tau w_{1}^{(2)}$, hence we only need to focus on these two elements. For simplicity, we define $\widetilde{D}_{1}^{(1)}=[3.1,3.9]$ as the region of $w_{1}^{(1)}$ and $\widetilde{D}_{1}^{(2)}=[9,+\infty)$ as the region of $\tau w_{1}^{(2)}$.

We firstly focus on the equation $\frac{\partial \widetilde{L}}{\partial w_{1}^{(2)}}=0$, which can be rewritten as:
\begin{align}
    \frac{4\alpha}{\tau}w_{1}^{(2)} 
    = h(w_{1}^{(1)}+\tau w_{1}^{(2)})+h(-w_{1}^{(1)}+\tau w_{1}^{(2)})\,.\label{original_eqw12}
\end{align}

For convenience, let $x=\tau w_{1}^{(2)}$. For any fixed $w_{1}^{(1)}\in \widetilde{D}_{1}^{(1)}$, we solve the following equation:
\begin{align}
    \varphi(x)\triangleq\frac{4\alpha}{\tau^2} x-(h(x+w_{1}^{(1)})+h(x-w_{1}^{(1)}))=0\label{eq_w_12_without noise}\,.
\end{align}

To solve the above equation, we analyze the monotonicity of $\varphi(x)$. When $(w_{1}^{(1)},x)\in \widetilde{D}_{1}^{(1)}\times \widetilde{D}_{1}^{(2)}$ and $\tau\ge 7$, it is clear that
\begin{align}
    \frac{\partial}{\partial x} (h(w_{1}^{(1)}+x)+h(-w_{1}^{(1)}+x))<0\,, \label{secondfact}
\end{align}
which means $\varphi'(x)>\frac{4\alpha}{\tau^2}>0$. Because $h(x)=\sigma'(x)\sigma(x)$ is a bounded function, it is clear that when $x\rightarrow +\infty$, $\varphi(x)\rightarrow +\infty.$ For any fixed $w_{1}^{(1)}\in \widetilde{D}_{1}^{(1)}$, it is clear that
\begin{align}
  \frac{36\alpha}{\tau^2} -(h(9+w_{1}^{(1)})+h(9-w_{1}^{(1)}))<0\notag\,.
\end{align}

The latest equation shows that $ \varphi(9)<0$. The above facts together with intermediate value principle imply that for any fixed $w_{1}^{(1)}\in \widetilde{D}_{1}^{(1)}$, there exists a unique $x\in \widetilde{D}_{1}^{(2)}$ which solves Eq. (\ref{eq_w_12_without noise}). Furthermore, the solutions $(w_{1}^{(1)},x)$ to Eq. (\ref{original_eqw12}) form a continuous curve in $\widetilde{D}_{1}^{(1)}\times \widetilde{D}_{1}^{(2)}$. 

Then, we solve $\frac{\partial \widetilde{L}}{\partial w_{1}^{(1)}}=0$ along the above curve. It is easy to check the following facts:
\begin{align}
    \frac{\partial \widetilde{L}}{\partial w_{1}^{(1)}}\bigg|_{w_{1}^{(1)}=3.1} < 0\,,\notag\\
    \frac{\partial \widetilde{L}}{\partial w_{1}^{(1)}}\bigg|_{w_{1}^{(1)}=3.9} > 0\,.\notag
\end{align}

By the above two equations, we can use intermediate value principle along the curve to get $(\widetilde{w}_{1}^{*(1)}, \widetilde{w}_{1}^{(2),*})$, which solves two equations $\frac{\partial \widetilde{L}}{\partial w_{1}^{(1)}}=0$ and $\frac{\partial \widetilde{L}}{\partial w_{1}^{(2)}}=0$ at the same time.

Using symmetry and similar process of $w_1$, it is easy to know that the equation $\frac{\partial \widetilde{L}}{\partial w_2} = 0$ has a solution $\widetilde{w}_{2}^{*}$, which satisfies $\widetilde{w}_{2}^{*}\in \widetilde{D}_2(\tau)$.
\end{proof}
~\\
\textbf{Lemma 2.} (restated) \textit{For $\alpha =1/800, \tau\ge 7$, there is a region $D_1^{B_0}(\tau)\times D_2^{B_0}(\tau)$ s.t. $\widetilde{L}$ is $2\alpha$-strongly convex and $(2\alpha+\tau^2+1.5)$-smooth.}
\begin{proof}[Proof]
At the beginning, we respectively define
\begin{align*}
    &D_1^{B_0}(\tau) = \{\vec{x}\in \mathbb{R}^d|x^{(1)}\in(2.3,4.7), \tau x^{(2)}\in(8.5,+\infty), x^{(k)}=(-\frac{3}{d^{0.49}},\frac{3}{d^{0.49}}), \forall k\in [3,d]\},\\
    &D_2^{B_0}(\tau) = \{\vec{x}\in \mathbb{R}^d|x^{(1)}\in(-4.7,-2.3), \tau x^{(2)}\in(8.5,+\infty), x^{(k)}=(-\frac{3}{d^{0.49}},\frac{3}{d^{0.49}}), \forall k\in [3,d]\}\,.
\end{align*}
It is clear that $\widetilde{H}=\frac{\partial^2 \widetilde{L}}{\partial w_1^2}\in \mathbb{R}^{d\times d}$ has the following form:
\begin{align}
\begin{bmatrix}
\frac{\partial^2 \widetilde{L}}{\partial w_{1}^{(1)2}} & \frac{\partial^2 \widetilde{L}}{\partial w_{1}^{(1)}w_{1}^{(2)}}&0&\cdots&0 &0\\
 \frac{\partial^2 \widetilde{L}}{\partial w_{1}^{(2)}w_{1}^{(1)}}&\frac{\partial^2 \widetilde{L}}{\partial w_{1}^{(2)2}} &0&\cdots&0 &0\\
 0&0&2\alpha&\cdots&0 &0\\
 \vdots & \vdots &\vdots &\ddots &\vdots &0\\
  0 & 0 &0 &\cdots &2\alpha &0\\
    0 & 0 &0 &\cdots &0 &2\alpha\\
\end{bmatrix}. \notag
\end{align}

Hence we only need to prove:
\begin{align}
    2\alpha I \preceq
    \begin{bmatrix}
    \frac{\partial^2 \widetilde{L}}{\partial w_{1}^{(1)2}} & \frac{\partial^2 \widetilde{L}}{\partial w_{1}^{(1)}w_{1}^{(2)}}\\
     \frac{\partial^2 \widetilde{L}}{\partial w_{1}^{(2)}w_{1}^{(1)}}&\frac{\partial^2 \widetilde{L}}{\partial w_{1}^{(2)2}} \\
    \end{bmatrix}
    \preceq (2\alpha+\tau^2+1.5) I\,.\notag
\end{align}

Direct calculation yields:
\begin{align*}
    \frac{\partial^2 \widetilde{L}}{\partial w_{1}^{(1)2}} &= -\frac{1}{2}(h'(w_{1}^{(1)})+h'(-w_{1}^{(1)})+h'(w_{1}^{(1)}+\tau w_{1}^{(2)})+h'(-w_{1}^{(1)}+\tau w_{1}^{(2)}))+2\alpha\,,\\
    \frac{\partial^2 \widetilde{L}}{\partial w_{1}^{(1)}w_{1}^{(2)}} &= \frac{\partial^2 \widetilde{L}}{\partial w_{1}^{(2)}w_{1}^{(1)}} = -\frac{\tau}{2}(h'(w_{1}^{(1)}+\tau w_{1}^{(2)})-h'(-w_{1}^{(1)}+\tau w_{1}^{(2)}))\,,\\
    \frac{\partial^2 \widetilde{L}}{\partial w_{1}^{(2)2}} &= -\frac{\tau^2}{2}(h'(w_{1}^{(1)}+\tau w_{1}^{(2)})+h'(-w_{1}^{(1)}+\tau w_{1}^{(2)}))+2\alpha\,.
\end{align*}

Thus, to prove  $ \frac{\partial^2 \widetilde{L}}{\partial w_1^2}\succeq \ 2\alpha I$,
we only need to prove the following two inequalities:
\begin{align}
    \frac{\partial^2 \widetilde{L}}{\partial w_{1}^{(1)2}} -2\alpha\geq 0\label{strongly convex without noise'}\,,
\end{align}
\begin{align}
    \begin{vmatrix}
    \frac{\partial^2 \widetilde{L}}{\partial w_{1}^{(1)2}} -2\alpha& \frac{\partial^2 \widetilde{L}}{\partial w_{1}^{(1)}w_{1}^{(2)}}\\
     \frac{\partial^2 \widetilde{L}}{\partial w_{1}^{(2)}w_{1}^{(1)}}&\frac{\partial^2 \widetilde{L}}{\partial w_{1}^{(2)2}} -2\alpha\\
    \end{vmatrix}\ge 0\,.\notag
\end{align}

The latest inequality is equivalent to:
\begin{align}
    4g'(w_{1}^{(1)}+\tau w_{1}^{(2)})h'(-w_{1}^{(1)}+\tau w_{1}^{(2)})+ (h'(w_{1}^{(1)}+\tau w_{1}^{(2)})+h'(-w_{1}^{(1)}+\tau w_{1}^{(2)})) (h'(w_{1}^{(1)})+h'(-w_{1}^{(1)}))\ge 0\,. \label{expression of strong convex without noise}
\end{align}

Eq. (\ref{strongly convex without noise'}) and Eq. (\ref{expression of strong convex without noise}) can be derived from the facts that $h'(x)<0$ when $x\geq 3.8$ and $h'(x)+h'(-x)<0$ when $x\in (2.3, 4.7)$.

To prove $ \frac{\partial^2 \widetilde{L}}{\partial w_1^2}\preceq \widetilde{L}_{m} I$, here $\widetilde{L}_{m}=2\alpha+\tau^2+1.5$, we need to prove the following two inequalities:
\begin{align}
 \widetilde{L}_{m}-\frac{\partial^2 \widetilde{L}}{\partial w_{1}^{(1)2}}\geq 0\,,\label{L-smooth without noise'}
\end{align}

\begin{align}
        \begin{vmatrix}
    \widetilde{L}_{m}-\frac{\partial^2 \widetilde{L}}{\partial w_{1}^{(1)2}} & \frac{\partial^2 \widetilde{L}}{\partial w_{1}^{(1)}w_{1}^{(2)}}\\
     \frac{\partial^2 \widetilde{L}}{\partial w_{1}^{(2)}w_{1}^{(1)}}&\widetilde{L}_{m}-\frac{\partial^2 \widetilde{L}}{\partial w_{1}^{(2)2}} \\
    \end{vmatrix}\ge 0\,.\label{L-smooth without noise}
\end{align}

The latest inequality is equivalent to :
\begin{align}
    \widetilde{L}_{m}-2\alpha \ge -\frac{1}{2}(h'(w_{1}^{(1)})+h'(-w_{1}^{(1)}))-\left(\frac{\tau^2+1}{2}\right)(h'(w_{1}^{(1)}+\tau w_{1}^{(2)})+h'(-w_{1}^{(1)}+\tau w_{1}^{(2)}))\,.
\end{align}

Eq. (\ref{L-smooth without noise'}) and Eq. (\ref{L-smooth without noise}) can be derived from the facts that $|h'|<\frac{1}{4}$ and $L>\tau^2$. Because the proof of $w_2$ is similar, we omit it and finish the proof of this lemma.
\end{proof}
~\\
\textbf{Lemma 3.} (restated) \textit{For $\tau = d^{\frac{1}{10}}$ and $n=d^2$, with probability $1-O\left(2^{-\frac{d^2}{10}}\right)$, $\left\|\frac{\partial \widehat{L}-\widetilde{L}}{\partial w_1}\Big|_{w_1=\widetilde{w}_1^{*}}\right\|_2\leq O(\tau n^{-\frac{9}{20}}), \left\|\frac{\partial^2\left(\widehat{L}-\widetilde{L}\right)}{\partial w_1^2}\right\|_F \leq O\left(\tau^2 n^{-\frac{9}{20}}\right)$ and $\|\nabla^3_1 \widehat{L}(w_1,w_2)\|_2\leq \tau^3$.}
\begin{proof}

Using Hoeffding type inequality, we know that
\begin{align}
    \text{Pr} \left(|n_l-\frac{n}{4}|\ge n^{\frac{11}{20}}\right)\leq 2^{-\frac{n}{10}}, \forall l\in [4]\,. \notag
\end{align}

Because $\sigma$, $\sigma'$, $\sigma''$ are  bounded, $\left|\frac{\partial^2 \widehat{L}}{\partial w_{1}^{(1)2}}-\frac{\partial^2 \widetilde{L}}{\partial w_{1}^{(1)2}}\right|\leq O\left(n^{-\frac{9}{20}}\right)$ holds with probability $1-O\left(2^{-\frac{d^2}{10}}\right)$. Using similar method to deal with $\frac{\partial^2 \widehat{L}}{\partial w_{1}^{(2)2}}, \frac{\partial^2 \widehat{L}}{\partial w_{1}^{(1)}w_{1}^{(2)}}$ and $\frac{\partial^2 \widehat{L}}{\partial w_{1}^{(2)}w_{1}^{(1)}}$, the noise terms for Hessian matrix can be written as
\begin{align}
\frac{\partial^2\left(\widehat{L}-\widetilde{L}\right)}{\partial w_1^2}=\begin{bmatrix}
O(n^{-\frac{9}{20}}) & \tau O(n^{-\frac{9}{20}})&0&\cdots&0 &0\\
\tau O(n^{-\frac{9}{20}})&\tau^2 O(n^{-\frac{9}{20}}) &0&\cdots&0 &0\\
 0&0&0&\cdots&0 &0\\
 \vdots & \vdots &\vdots &\ddots &\vdots &0\\
  0 & 0 &0 &\cdots &0 &0\\
    0 & 0 &0 &\cdots &0 &0\\
\end{bmatrix}. \notag
\end{align}

Using the similar process to deal with $\widehat{g}(w_1,w_2)$ and $\nabla^3_1 \widehat{L}(w_1,w_2)$, $\left\|\frac{\partial \widehat{L}}{\partial w_1}\Big|_{w_1=\widetilde{w}_1^{*}}\right\|_2\leq O\left(\tau n^{-\frac{9}{20}}\right)$ and $ \|\nabla^3_1 \widehat{L}(w_1,w_2)\|_2 \leq \tau^3$ hold with probability $1-O\left(2^{-\frac{d^2}{10}}\right)$.
\end{proof}
~\\
\textbf{Lemma 4.} (restated) \textit{When $w_1\in D_1^{B_0}(\tau), \tau = d^{\frac{1}{10}}, \rho= 1/d^{1.5}$ and $n=d^2$, with probability $1-O\left(e^{-d^{\frac{1}{10}}}\right)$ and large enough $d$, $\left\|\frac{\partial L-\widehat{L}}{\partial w_1}\Big|_{w_1=\widetilde{w}_1^{*}}\right\|_2\leq O(\rho^{\frac{13}{15}}d^{\frac{6}{10}}), \left\|\frac{\partial^2\left(L-\widehat{L}\right)}{\partial w_1^2}\right\|_F\leq O(\rho^{\frac{4}{5}} d^\frac{11}{10})$ and $\|\nabla^3_1 L(w_1, w_2)\|_2\leq \Theta(\sqrt{d})$.}
\begin{proof}
First, we calculate the upper bound of Hessian noise $\left\|\frac{\partial^2\left(L-\widehat{L}\right)}{\partial w_1^2}\right\|_F$ and take $\frac{\partial^2 L}{\partial w_{1}^{(1)2}}$ as an example. Similar to Eq. (\ref{gradient}), $\frac{\partial^2 L}{\partial w_{1}^{(1)2}}$ can be divided into four parts $\frac{\partial^2 L}{\partial w_{1}^{(1)2}}=A+B+C+D$ and each part corresponding to one type of datapoint. Because this section only focuses on 
the noise terms in the transition from $\widehat{L}$ to $L$, the analysis of $A, B, C, D$ are similar. For simplicity, we only show the first term $A$:
\begin{align}
  A= \sum_{i=1}^{n_1}\mathbb{E}_{\xi_{\text{aug}},\xi'_{\text{aug}}}& \Big[(1+\rho\xi_i^{(1)}+\xi_{\text{aug}}^{\prime(1)})^2\sigma\left(w_{1}^{(1)}+w_1^{\top}(\rho\xi_i+\xi_{\text{aug}})\right)\sigma''\left(w_{1}^{(1)}+ w_1^{\top}(\rho\xi_i+\xi^{\prime}_{\text{aug}})\right)\notag\\&\quad+(1+\rho\xi_i^{(1)}+\xi_{\text{aug}}^{(1)})^2(\sigma''\left(w_{1}^{(1)}+w_1^{\top}(\rho\xi_i+\xi_{\text{aug}})\right)\sigma\left(w_{1}^{(1)}+ w_1^{\top}(\rho\xi_i+\xi^{\prime}_{\text{aug}})\right)\notag\\
    &\quad+(1+\rho\xi_i^{(1)}+\xi_{\text{aug}}^{(1)})(1+\rho\xi_i^{(1)}+\xi_{\text{aug}}^{\prime(1)})
    \sigma'\left(w_{1}^{(1)}+w_1^{\top}\xi_{\text{all},i}\right)\sigma'\left(w_{1}^{(1)}+ w_1^{\top}(\rho\xi_i+\xi^{\prime}_{\text{aug}})\right)\Big]\,
    . \notag
\end{align}

By the Lagrange’s Mean Value Theorem, the terms in the latest equation can be rewritten as follows: 
\begin{align*}
&\sigma\left(w_{1}^{(1)}+w_1^{\top}(\rho\xi_i+\xi_{\text{aug}})\right)=\sigma(w_{1}^{(1)})+\sigma'(\theta_1)\left( w_1^{\top}(\rho\xi_i+\xi_{\text{aug}})\right)\,,\\&\sigma'\left(w_{1}^{(1)}+ w_1^{\top}(\rho\xi_i+\xi_{\text{aug}})\right)=\sigma'(w_{1}^{(1)})+\sigma''(\theta_2)\left(w_1^{\top}(\rho\xi_i+\xi_{\text{aug}})\right)\,,\\
    &\sigma''\left(w_{1}^{(1)}+ w_1^{\top}(\rho\xi_i+\xi_{\text{aug}})\right)=\sigma''(w_{1}^{(1)})+\sigma'''(\theta_3)\left(w_1^{\top}(\rho\xi_i+\xi_{\text{aug}})\right)\,.
\end{align*}

From the above facts, we can get a general bound for noise terms:
\begin{align}
\Big|\frac{\partial^2 (L-\widehat{L})}{\partial w_{1}^{(k_1)}\partial w_{1}^{(k_2)}}\Big|\leq&
\frac{C\tau^2}{n}\sum_{i=1}^n \mathbb{E}_{\xi_{\text{aug}},\xi'_{\text{aug}}}\bigg[\left|w_1^{\top}(\rho\xi_i+\xi_{\text{aug}})\right|+\left|w_1^{\top}(\rho\xi_i+\xi^{\prime}_{\text{aug}})\right|+\left|w_1^{\top}(\rho\xi_i+\xi^{\prime}_{\text{aug}})\right|\left|w_1^{\top}(\rho\xi_i+\xi_{\text{aug}})\right|\notag\\&\qquad\qquad\qquad\qquad+\left|\rho\xi_i^{(1)}+\xi_{\text{aug}}^{(1)}\right|^2+\left|\rho\xi_i^{(1)}+\xi_{\text{aug}}^{\prime(1)}\right|^2+\left|\rho\xi_i^{(1)}+\xi_{\text{aug}}^{(1)}\right|+\left|\rho\xi_i^{(1)}+\xi_{\text{aug}}^{\prime(1)}\right|\bigg]\notag\\
\leq& \frac{2C\tau^2}{n}\sum_{i=1}^n \mathbb{E}_{\xi_{\text{aug}}}\left[|w_1^{\top}(\rho\xi_i+\xi_{\text{aug}})|+|w_1^{\top}(\rho\xi_i+\xi_{\text{aug}})|^2+|\rho\xi_i^{(1)}+\xi_{\text{aug}}^{(1)}|^2+|\rho\xi_i^{(1)}+\xi_{\text{aug}}^{(1)}|\right]\notag\\
\leq& \frac{2C\tau^2}{n}\sum_{i=1}^n \mathbb{E}_{\xi_{\text{aug}}}\left[|w_1^{\top}\xi_{\text{aug}}|
+\rho |w_1^{\top}\xi_i|+|w_1^{\top}\xi_{\text{aug}}+\rho w_1^{\top}\xi_i|^2
+|\xi_{\text{aug}}^{(1)}+\rho\xi_i^{(1)}|^2+|\xi_{\text{aug}}^{(1)}|+|\rho\xi_i^{(1)}|\right]\notag\\
\leq&2C\tau^2\mathbb{E}_{\xi_{\text{aug}}}\Big[|w_1^{\top}\xi_{\text{aug}}|+|w_1^{\top}\xi_{\text{aug}}|^2+|\xi_{\text{aug}}^{(1)}|+\xi_{\text{aug}}^{(1)2}\Big]+\frac{2C\tau^2}{n}\sum_{i=1}^n (|\rho w_1^{\top}\xi_i|+|\rho w_1^{\top}\xi_i|^2+|\rho\xi_i^{(1)}|+|\rho\xi_i^{(1)}|^2)\notag\\
=&O(\tau^3\rho)+\frac{2C\tau^2}{n}\sum_{i=1}^n \left(|\rho w_1^{\top}\xi_i|+|\rho w_1^{\top}\xi_i|^2+|\rho\xi_i^{(1)}|+|\rho\xi_i^{(1)}|^2\right), \forall k_1, k_2\in[d]\,,\notag
\end{align}
where $C$ is a constant. The first equality is by the fact that $\sigma, \sigma', \sigma'',\sigma'''$
are all bounded, the second inequality is by Cauchy-Schwarz inequality, the third inequality is by triangle inequality.

For a standard Gaussian random variable $\xi$, we have: 
\begin{align}
    \text{Pr}\left(|\xi|\leq d^{\frac{1}{10}}\right)\ge d^{-\frac{1}{10}}e^{-\frac{d^{-\frac{1}{5}}}{2}}\,.\notag
\end{align}

From above tail bound and the union bound, $ |\xi_i^{(1)}|\leq d^{\frac{1}{10}}, \left|w_1^{\top}\xi_i\right|\leq \|w_1\|_2d^{\frac{1}{10}} , \forall i\in[n]$ holds with probability $1-nd^{-\frac{1}{10}}e^{-\frac{d^{-\frac{1}{5}}}{2}}$ where $n=d^2$. Because the error probability is exponential in $d$, the error probability $nd^{-\frac{1}{10}}e^{-\frac{d^{-\frac{1}{5}}}{2}} \leq e^{-d^{\frac{1}{10}}}$ for large enough $d$. Then the above noise terms bounds hold with probability $1-O(e^{-d^{\frac{1}{10}}})$.

Because of $w_1\in D_1(\tau)$, the right hand of Eq. (\ref{original_eqw12})$\leq O(1)$ and $\rho=1/d^{1.5}$, it is clear that $w_{1}^{(1)} = O(1), w_{1}^{(2)}= O(\tau), |w_{1}^{(k)}|\leq \frac{1}{d^{0.49}}, \forall k\in [3,d]$ and $\|w_1\|_2=O(\tau)$. Then we can get the upper bound for $\left|\frac{\partial^2 (L-\widehat{L})}{\partial w_{1}^{(k_2)}\partial w_{1}^{(k_2)}}\right|$:

\begin{align*}
\left|\frac{\partial^2 (L-\widehat{L})}{\partial w_{1}^{(k_1)}\partial w_{1}^{(k_2)}}\right|\leq&
O(\tau^3\rho)+\frac{2C\tau^2}{n}\sum_{i=1}^n \left(|\rho w_1^{\top}\xi_i|+|\rho w_1^{\top}\xi_i|^2+|\rho\xi_i^{(1)}|+|\rho\xi_i^{(1)}|^2\right)\\
=&O(\tau^3\rho)+O(\tau^3\rho d^{\frac{1}{10}})+O(\tau^4\rho^2 d^{\frac{1}{5}})+O(\tau^2\rho d^{\frac{1}{10}})+O(\tau^2\rho^2 d^{\frac{1}{5}})\\
=&O(\tau^3\rho d^{\frac{1}{10}})=O(\rho^{\frac{4}{5}} d^{\frac{1}{10}}), \forall k_1, k_2\in[d]\,,
\end{align*}
and therefore $\left\|\frac{\partial^2\left(L-\widehat{L}\right)}{\partial w_1^2}\right\|_F^2\leq O(d^2\rho^{\frac{8}{5}} d^\frac{1}{5})$.

Similar to the above proof process, we can prove that with probability $1-O(e^{-d^{\frac{1}{10}}})$, $\left\|\frac{\partial^3 (L-\widehat{L})}{\partial w_{1}^{(k_1)}\partial w_{1}^{(k_2)} \partial w_{1}^{(k_3)}}\right\|_2=O(\tau^4\rho d^{\frac{1}{10}}),\forall k_1, k_2, k_3\in [d]$ which means $\|\nabla^3_1(L-\widehat{L})(w_1, w_2)\|_2=O(\tau^4\rho d^{\frac{1}{10}}d^{\frac{3}{2}})$. Lemma \ref{lem:dislocation_noise} already shows that $\|\nabla^3_1(\widehat{L})(w_1, w_2)\|_2=O(\tau^3)$ holds with probability $1-O(2^{-\frac{d^2}{10}})$. Hence, $\|\nabla_1^3L(w_1,w_2)\|_2\leq O(\tau^4\rho d^{\frac{1}{10}}d^{\frac{3}{2}}+\tau^3)=O(\tau^4\rho d^{\frac{1}{10}}d^{\frac{3}{2}})=O(\sqrt{d})$ holds  with probability $(1-O(e^{-d^{\frac{1}{10}}}))(1-O(2^{-\frac{d^2}{10}}))=1-O(e^{-d^{\frac{1}{10}}})$.

By the fact that $\left\|\frac{\partial (L-\widehat{L})}{\partial w_{1}^{(k)} }\right\|_2=O(\tau^2\rho d^{\frac{1}{10}})$, it is clear that $\left\|\frac{\partial \left(L-\widehat{L}\right)}{\partial w_1}\Big|_{w_1=\widetilde{w}_1^{*}}\right\|_2\leq O(\sqrt{d}\rho^{\frac{13}{15}}d^{\frac{1}{10}})$ holds with probability $1-O(e^{-d^{\frac{1}{10}}})$.
\end{proof}
~\\
\textbf{Lemma 5.} (restated) \textit{For $\tau= d^{\frac{1}{10}}, \rho= \frac{1}{d^{1.5}}$ and $n=d^2$, when $(w_1,w_2)\in D_1^{B_0}(\tau)\times D_2^{B_0}(\tau)$, with probability $1-O\left(e^{-d^{\frac{1}{10}}}\right)$ and large enough $d$, $L$ is $(2\alpha-\rho^{\frac{4}{5}}d^{\frac{11}{10}})$-strongly convex and $(2\alpha+\tau^2+1.5+\rho^{\frac{4}{5}}d^{\frac{11}{10}})$-smooth. At the same time, $\nabla^2 L(w_1)$ is $L_H$-Lipschitz continuous Hessian where $L_H=\Theta(\sqrt{d})$.}

\begin{proof}[Proof] Define $H(w_1)= \widetilde{H}(w_1)+\frac{\partial^2\left(\widehat{L}-\widetilde{L}\right)}{\partial w_1^2}+\frac{\partial^2\left(L-\widehat{L}\right)}{\partial w_1^2}$ as the Hessian matrix of $L$ and the eigenvalues of $H(w_1)$ as $\{\lambda_k^{H(w_1)}\}_{k=1}^d$.
Lemma \ref{lem:strongconvex} shows that $\widetilde{H}(w_1)$ is $2\alpha$-strongly convex. Hence all eigenvalues $\{\lambda_k^{\widetilde{H}(w_1)}\}_{k=1}^d$ of $\widetilde{H}(w_1)$ are larger than $2\alpha$. Using the matrix eigenvalue perturbation theory, Lemma \ref{lem:noiseeig} shows that  
\begin{align}
    \sum_{k=1}^d(\lambda_k^{\widetilde{H}(w_1)}-\lambda_{\pi(k)}^{H(w_1)})^2 \leq \sqrt{2}\left\|\frac{\partial^2\left(\widehat{L}-\widetilde{L}\right)}{\partial w_1^2}+\frac{\partial^2\left(L-\widehat{L}\right)}{\partial w_1^2}\right\|_F^2\leq  O(\rho^{\frac{8}{5}} d^\frac{11}{5}+\tau^4n^{-\frac{9}{10}})=O(\rho^{\frac{8}{5}} d^\frac{11}{5})\,. \notag
\end{align}

Hence $L$ is $(2\alpha-\rho^{\frac{4}{5}}d^{\frac{11}{10}})$-strongly convex and $(2\alpha+\tau^2+1.5+\rho^{\frac{4}{5}}d^{\frac{11}{10}})$-smooth.

For $L_H$-Lipschitz continuous Hessian, Lemma \ref{lem:data_noise} shows an upper bound
$L_H\leq\|\nabla_1^3L(w_1,w_2)\|_2\leq O(\sqrt{d})$. Because the proof of $w_2$ is similar, we omit it and finish the proof of this lemma.
\end{proof}
~\\
\textbf{Lemma 6.} (restated) \textit{ Let $(\widetilde{w}_1^*, \widetilde{w}_2^*)$ be the solution in Lemma \ref{lem:existence_no_noise} and $W=[w_1, w_2]\in \mathbb{R}^{2\times d}$. When $\|w_1-\widetilde{w}_1^*\|_2\leq d^{-\frac{1}{2}}$ and $\|w_2-\widetilde{w}_2^*\|_2\leq d^{-\frac{1}{2}}, |\Pi_W e_j|\geq 1-O(\tau^3 d^{-\frac{1}{2}}), \forall j\in [2].$}

\begin{proof}

We define  $\widetilde{W}^*=[\widetilde{w}_1^{*},\widetilde{w}_2^{*}]^{\top}\in \mathbb{R}^{2\times d}, W=[w_1,w_2]^{\top}\in \mathbb{R}^{2\times d}$ and $\delta B=WW^{\top}-\widetilde{W}^*\widetilde{W}^{*\top}$. For $j\in [2]$, we define $ \delta b_j=(\widetilde{W}^*-W)^{\top}e_j, \widetilde{x}_j=(\widetilde{W}^*\widetilde{W}^{*\top})^{-1}\widetilde{W}^*e_j, x_j=(WW^{\top})^{-1}We_j$ and $\delta x_j=\widetilde{x}_j-x_j$. It is easy to check $\|(WW^{\top})^{-1}\|_2=O(1)$, $\|\delta b_j\|_2=O(d^{-\frac{1}{2}})$, $\|\delta B\|_2=O(\tau d^{-\frac{1}{2}})$ and $\|\widetilde{x}_j\|_2=O(\tau), \forall j\in[2]$. It is clear that
\begin{align*}
(\widetilde{W}^*\widetilde{W}^{*\top}+\delta B)(\widetilde{x}_j+\delta x_j)=\widetilde{W}^*e_j+\delta, \forall j\in[2] \,,
\end{align*}
so
\begin{align*}
\|\delta \widetilde{x}_j\|_2\leq \|(WW^{\top})^{-1}\|_2(\|\delta b_j\|_2+\|\delta B\|_2\|\widetilde{x}_j\|_2)=O(\tau^2 d^{-\frac{1}{2}})\,.
\end{align*}

Now we can calculate the length of the projection of $e_j$ onto the plane spanned by $w_1$ and $w_2:$ 
\begin{align}
    |\Pi_W e_j|=(e_j^{\top}W^{\top}x_j)^{\frac{1}{2}}=&(e_j^{\top}\widetilde{W}^{*\top}\widetilde{x}_j+e_j^{\top}W^{\top}x_j-e_j^{\top}\widetilde{W}^{*\top}\widetilde{x}_j)^{\frac{1}{2}}\notag\\
    =&(1+e_j^{\top}W^{\top}x_j-e_j^{\top}\widetilde{W}^{*\top}\widetilde{x}_j)^{\frac{1}{2}}\notag\\
    \ge&(1-|e_j^{\top}W^{\top}x_j-e_j^{\top}\widetilde{W}^{*\top}\widetilde{x}_j|)^{\frac{1}{2}}\notag\\
    =&(1-|(e_j^{\top}\widetilde{W}^{*\top}+\delta b^{\top}_j)(\widetilde{x}_j+\delta x_j)-e_j^{\top}\widetilde{W}^{*\top}\widetilde{x}_j|)^{\frac{1}{2}}\notag\\
    \geq&(1-|(e_j^{\top}\widetilde{W}^{*\top}\delta x_j|-|\delta b^{\top}(\widetilde{x}_j+\delta x_j)|)^{\frac{1}{2}}\notag\\
    =&(1-O(\tau^3d^{-\frac{1}{2}}))^\frac{1}{2}\notag\\
    \geq&1-O(\tau^3d^{-\frac{1}{2}})\,. \notag
\end{align}
\end{proof}

\begin{lemma}[Eigenvalue perturbation lemma ~\cite{kahan1975spectra}]\label{lem:noiseeig}
$A\in \mathbb{R}^{n\times n}$ is a Hermite matrix, $B =A+N$ is a matrix induced by $A$ ($N$ is a noise matrix). The eigenvalues of $A$ is $\lambda(A)=\{\lambda_k\}_{k\in[n]}$, $\lambda(B) =\{\mu_k\}_{k\in[n]}$, then
\begin{align}
    \sum_{i=1}^n(\lambda_i-\mu_{\pi(i)})^2\leq \sqrt{2}\|A-B\|_F^2\,.\notag
\end{align}
\end{lemma}

\begin{lemma}[Convergence lemma ~\cite{bubeck2015convex}]\label{lem:convergence}
 Let $f$ be locally $\mu$-strongly convex and $L_m$-smooth, if $\eta_t=\eta=\frac{2}{\mu+L_m}$, $\kappa = \frac{L_m}{\mu}$, and $x^*\in \text{argmin}_{x\in \mathcal{X}}f(x)$, then
\begin{align*}
    \|x^t-x^*\|_2\leq\left(\frac{\kappa-1}{\kappa+1}\right)^t\|x^{(0)}-x^*\|_2\,.
\end{align*}
\end{lemma}

\end{document}